\documentclass{article}

\usepackage{microtype}
\usepackage{graphicx}
\usepackage{subcaption}
\usepackage{booktabs} 

\usepackage{tabularx}
\usepackage{placeins}

\usepackage{array} 

\newcolumntype{L}{>{\raggedright\arraybackslash}X}
\newcolumntype{C}{>{\centering\arraybackslash}X}
\newcolumntype{R}{>{\raggedleft\arraybackslash}X}

\usepackage{multirow}

\usepackage{hyperref}

\usepackage[preprint]{icml2026}

\usepackage{amsmath}
\usepackage{amssymb}
\usepackage{mathtools}
\usepackage{amsthm}
\usepackage{comment}

\usepackage[capitalize,noabbrev]{cleveref}

\theoremstyle{plain}
\newtheorem{theorem}{Theorem}[section]

\theoremstyle{definition}

\theoremstyle{remark}

\usepackage[textsize=tiny]{todonotes}

\newcommand{\draft}[1]{}

\icmltitlerunning{Submission and Formatting Instructions for ICML 2026}

\begin{document}

\twocolumn[
  \icmltitle{Spectra: Rethinking Optimizers for LLMs Under Spectral Anisotropy}



  \icmlsetsymbol{equal}{*}

  \begin{icmlauthorlist}
    \icmlauthor{Zhendong Huang}{equal,fdu}
    \icmlauthor{Hengjie Cao}{equal,fdu}
    \icmlauthor{Fang Dong}{fdu}
    \icmlauthor{Ruijun Huang}{fdu}
    \icmlauthor{Mengyi Chen}{fdu}
    \icmlauthor{Yifeng Yang}{fdu}
    \icmlauthor{Xin Zhang}{fdu}
    \icmlauthor{Anrui Chen}{fdu}
    \icmlauthor{Mingzhi Dong}{bath}
    \icmlauthor{Yujiang Wang}{oxford}
    \icmlauthor{Jinlong Hou}{sii}
    \icmlauthor{Qin Lv}{boulder}
    \icmlauthor{Robert P. Dick}{michigan}
    \icmlauthor{Yuan Cheng}{sii}
    \icmlauthor{Fan Yang}{fdu}
    \icmlauthor{Tun Lu}{fdu}
    \icmlauthor{Li Shang}{fdu}
  \end{icmlauthorlist}

  \icmlaffiliation{fdu}{Fudan University, Shanghai, China}
  \icmlaffiliation{bath}{University of Bath, Bath, United Kingdom}
  \icmlaffiliation{oxford}{Oxford Suzhou Centre for Advanced Research, Suzhou, China}
  \icmlaffiliation{sii}{Shanghai Innovation Institute, Shanghai, China}
  \icmlaffiliation{boulder}{Department of Computer Science, University of Colorado Boulder, Colorado, USA}
  \icmlaffiliation{michigan}{Department of Electrical Engineering and Computer Science, University of Michigan}

  \icmlcorrespondingauthor{Li Shang}{lishang@fudan.edu.cn}

  \icmlkeywords{Machine Learning, ICML}

  \vskip 0.3in
]



\printAffiliationsAndNotice{}  

\begin{abstract}
Gradient signals in LLM training are highly anisotropic: recurrent linguistic structure concentrates energy into a small set of dominant spectral directions, while context-specific information resides in a long tail. We show that this spike–tail separation persists throughout training, with the spike occupying only about $1.5\%$ of directions yet dominating optimizer statistics. This dominance suppresses tail learning by contracting tail updates through second-moment normalization and tightening the globally stable learning-rate bound.
Motivated by this analysis, we propose \textit{Spectra}, a spike-aware optimizer that suppresses the dominant low-rank spike subspace without amplifying the noise-sensitive spectral tail. Spectra tracks the spike subspace via cached, warm-started power iteration and applies low-rank spectral shaping with negligible overhead and substantially reduced optimizer-state memory.
On LLaMA3-8B trained on 50B tokens, Spectra reaches the same target loss $30\%$ faster than AdamW, reduces per-step end-to-end overhead by $0.7\%$, cutting optimizer-state memory by $49.25\%$, and improves average downstream accuracy by $1.62\%$.
Compared to Muon, Spectra is $5.1\times$ faster in optimizer processing time, achieves a lower final loss, and improves average accuracy by $0.66\%$. 
\end{abstract}

\section{Introduction}

Training on natural language corpora produces a highly imbalanced learning signal: grammatical and functional patterns recur ubiquitously across data, whereas semantically rich world knowledge is distributed over a vast long tail of rare and sparsely sampled events~\cite{piantadosi2014zipf,linders2023zipf,mikhaylovskiy2025zipf}. This imbalance induces strong directional correlations in representation space, resulting in highly anisotropic contextual embeddings~\cite{arora2017simple,mu2017all, ethayarajh2019contextual, li2020sentence, timkey2021all}. Consequently, gradients associated with common linguistic structure are repeatedly reinforced during training, while gradients corresponding to long-tail semantic content remain weaker, and more intermittent~\cite{kandpal2023large}.

This work aims to characterize the anisotropic structure of gradient signals in LLM training and translate it into optimizer design principles.
Our analysis is grounded in a spectral perspective because the imbalance is \emph{directional rather than element-wise}: individual coordinates entangle multiple latent factors and obscure independent learning modes, whereas spectral analysis reveals how skewed signals concentrate into a small set of correlated directions~\cite{cao2025metis,ethayarajh2019contextual,timkey2021all}.
Our analysis yields four key observations:

\noindent\textbf{Observation 1: A common-structure low-rank spike with a smooth semantic tail.}
The gradient spectrum exhibits a pronounced low-rank spike: roughly the top 1.5\% directions carry a disproportionate fraction of gradient energy and are separated from the tail spectrum by one to two orders of magnitude.
This anisotropy phenomenon is consistent across model scales, modules, and training stages, and admits a two-region semantic correspondence: the spike is primarily driven by common linguistic structure, while the tail encodes finer, context-specific semantic variations.

\noindent\textbf{Observation 2: Spike updating suppresses long-tail learning.}
Spike directions dominate AdamW’s~\cite{kingma2014adam} second-moment accumulation, so element-wise normalization is effectively set by the spike subspace and contracts tail update magnitudes.
Moreover, spike-dominated gradient variance bounds the optimal learning rate, imposing a conservative global step size that further limits progress in long-tail directions.

\noindent\textbf{Observation 3: Smaller singular directions carry sparser semantics and higher statistical relative variance.}
Along the spectral tail, semantic signals become increasingly sparse and intermittent: only a diminishing fraction of samples yield non-negligible projections onto smaller-singular directions.
Accordingly, their relative variance rises sharply, meaning stochastic fluctuations dominate as singular values decrease.
As a result, updates in these small-singular directions are increasingly unstable and easily drowned out.

\noindent\textbf{Observation 4: Numerical variance further destabilizes the spectral tail under iterative updates.}
Spectrum-aware optimizers, such as Muon~\cite{jordan6muon}, often implement spectral processing via iterative routines, such as Newton--Schulz iteration~\cite{higham1997stable}, which introduce numerical perturbations that concentrate in small-singular components, amplifying tail disturbances; aggressive tail equalization further worsens this while adding substantial computation.

\noindent \textbf{Spectra: design and properties.}
Our analysis motivates a clear design principle:
\emph{suppress the dominant low-rank spike subspace while avoiding aggressive amplification of the fragile spectral tail}.
Guided by this principle, \textit{Spectra} is designed with the following properties:

\noindent\emph{Efficiency.} Spectra tracks only the low-dimensional spike subspace using intermittently updated, warm-started power iteration and operates exclusively on this fixed small-rank component. As a result, it incurs low computational overhead and avoids storing per-parameter second-order statistics, substantially reducing optimizer-state memory.

\noindent\emph{Optimization.} By selectively attenuating spike-dominated updates, Spectra prevents common features from dominating optimization dynamics and avoids amplifying noise-dominated spectral components. This relaxes the spike-induced gradient-variance constraint on stable learning rates, widening the effective learning-rate range and accelerating convergence.

\noindent \emph{Parallelism.} Low-rank spike subspace estimation is naturally distributed-friendly: power iteration can be implemented via local GEMMs with only lightweight collectives on low-rank quantities. This makes Spectra well suited for large-scale parallel and distributed training without requiring full-gradient synchronization.

We evaluate \textit{Spectra} on LLaMA3-8B trained on 50B tokens, comparing against AdamW and Muon.

\noindent\textbf{Compared to AdamW}, Spectra reaches the same target loss $30\%$ faster in wall-clock time, reduces optimizer overhead (–$0.7\%$ end-to-end cost, –$49.25\%$ optimizer state memory), and improves average downstream accuracy by $1.62\%$. These gains stem from suppressing the dominant common-structure spike that hinders long-tail semantic learning and avoiding dense per-parameter second-moment estimation.

\noindent\textbf{Compared to Muon}, Spectra is $5.1\times$ faster in optimizer processing time, achieves a lower final loss, and improves average accuracy by $0.66\%$, 
while applying a localized spectral intervention that targets only the dominant spike subspace and avoids global spectrum flattening.

\section{Analysis}
\label{sec:analysis}

\subsection{Gradient Anisotropy: A consistent characteristic}

For a gradient matrix $\mathbf{G} \in \mathbb{R}^{m \times n}$, Singular Value Decomposition (SVD) is applied to obtain singular values $\{\sigma_i\}_{i=1}^{\min(m,n)}$, left (right) singular vectors $\{\mathbf{u}_i\} \in \mathbb{R}^m$ ($\{\mathbf{v}_i\} \in \mathbb{R}^n$), such that 
\(
\mathbf{G} = \sum_{i=1}^{\min(m,n)} \sigma_i \mathbf{u}_i \mathbf{v}_i^\top.
\)
We assume singular values are sorted in descending order, i.e., $\sigma_1 \geq \sigma_2 \geq \dots \geq \sigma_r \ge 0$ with $r = \min(m,n)$.

\begin{figure*}[ht]
    \centering
    \includegraphics[width=\textwidth]{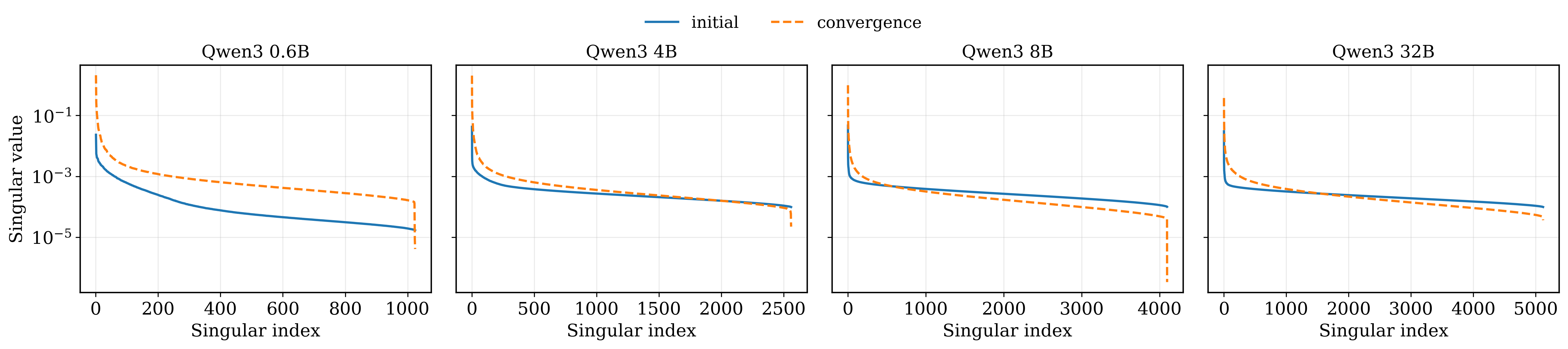}
    \caption{Singular-value spectra of the deepest-layer MLP gradient in Qwen3 models (0.6B--32B) at multiple training stages exhibit a consistent ``low-rank spike + smooth tail'' profile, with spike singular values separated from the tail by $\sim$1--2 orders of magnitude and occupying a nearly constant $\approx 1.5\%$ of directions.}
    \label{fig:gradient-anisotropy}
\end{figure*}

Across Qwen3~\cite{yang2025qwen3} models of different scales, from 0.6B to 32B, and at different training stages, Figure~\ref{fig:gradient-anisotropy} reports the gradient singular spectrum of the deepest MLP layer.
Results for attention modules and shallower layers are provided in Appendix~\ref{appendix:gradient-anisotropy}.
In all cases, \emph{the gradient spectrum exhibits a consistent anisotropic pattern}, following a ``low-rank spike + smooth tail'' profile: a compact spike block is separated from the tail by roughly one to two orders of magnitude in singular value.
In practice, the spectral region preceding the first eigengap occupies a small and stable fraction of directions, approximately $1.5\%$ across model scales, we therefore adopt this value as a stable default.

\subsection{Linguistic Correspondence of Gradient Anisotropy}
This subsection attributes gradient anisotropy to different linguistic signals in the training data. Specifically, we analyze gradient spectra on Qwen3-0.6B~\cite{yang2025qwen3} and LLaMA3-8B~\cite{dubey2024llama}.
For each model, we compute the gradient matrix $\mathbf{G}$ under three controlled conditions:
(i) \emph{Raw}, serving as the unmodified reference;
(ii) \emph{FreqNorm}, reducing the excessive contribution of high-frequency tokens;
and (iii) \emph{Shuffle}, removing syntactic dependencies and sequential structure.
For \emph{FreqNorm}, with token $t_j$ at position $j$ and corpus frequency $f(t_j)$, we rescale the token-wise loss as
$\tilde{\ell}_j=\ell_j/f(t_j)$ and backpropagate from $\sum_j \tilde{\ell}_j$.
For \emph{Shuffle}, we randomly permute tokens within each sentence before computing gradients.
We then compare the singular spectra of $\mathbf{G}$ across conditions to localize which spectral regions are sensitive to frequency skew and syntactic-order perturbations.

\begin{figure}[ht]
    \captionsetup{skip=1pt}
    \centering
    \includegraphics[width=\linewidth]{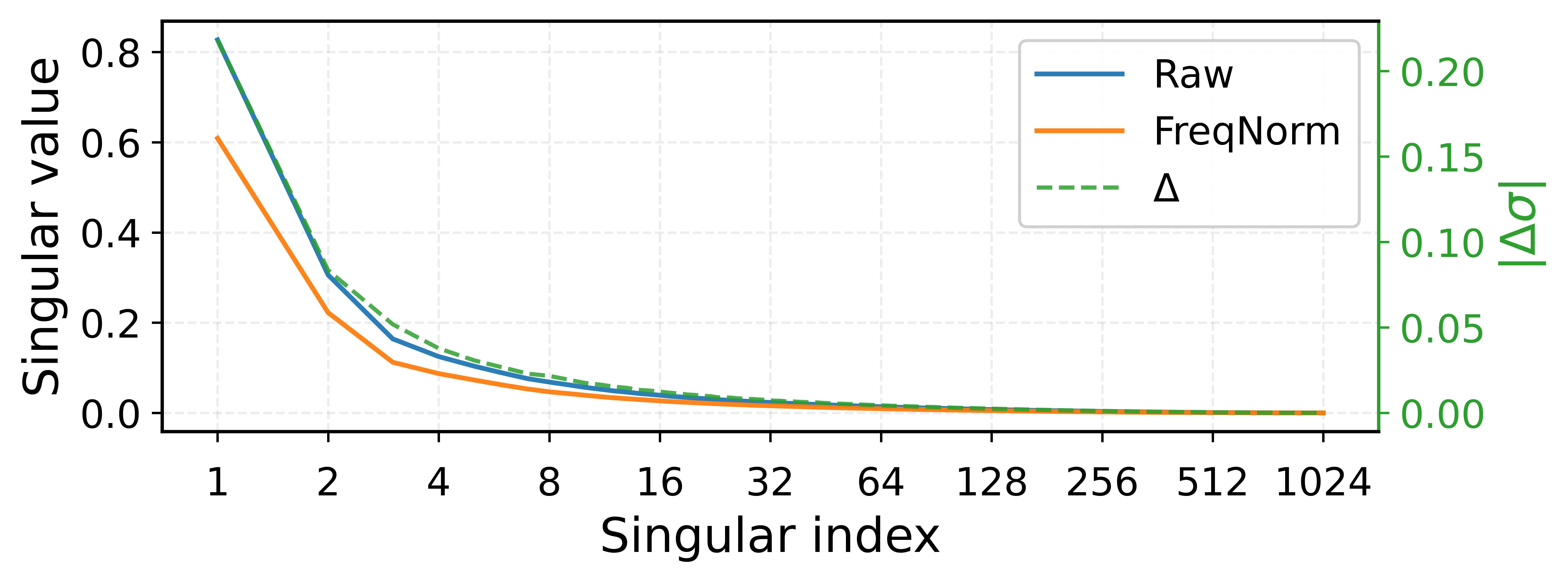}
    \vspace{0.6em}
    \includegraphics[width=\linewidth]{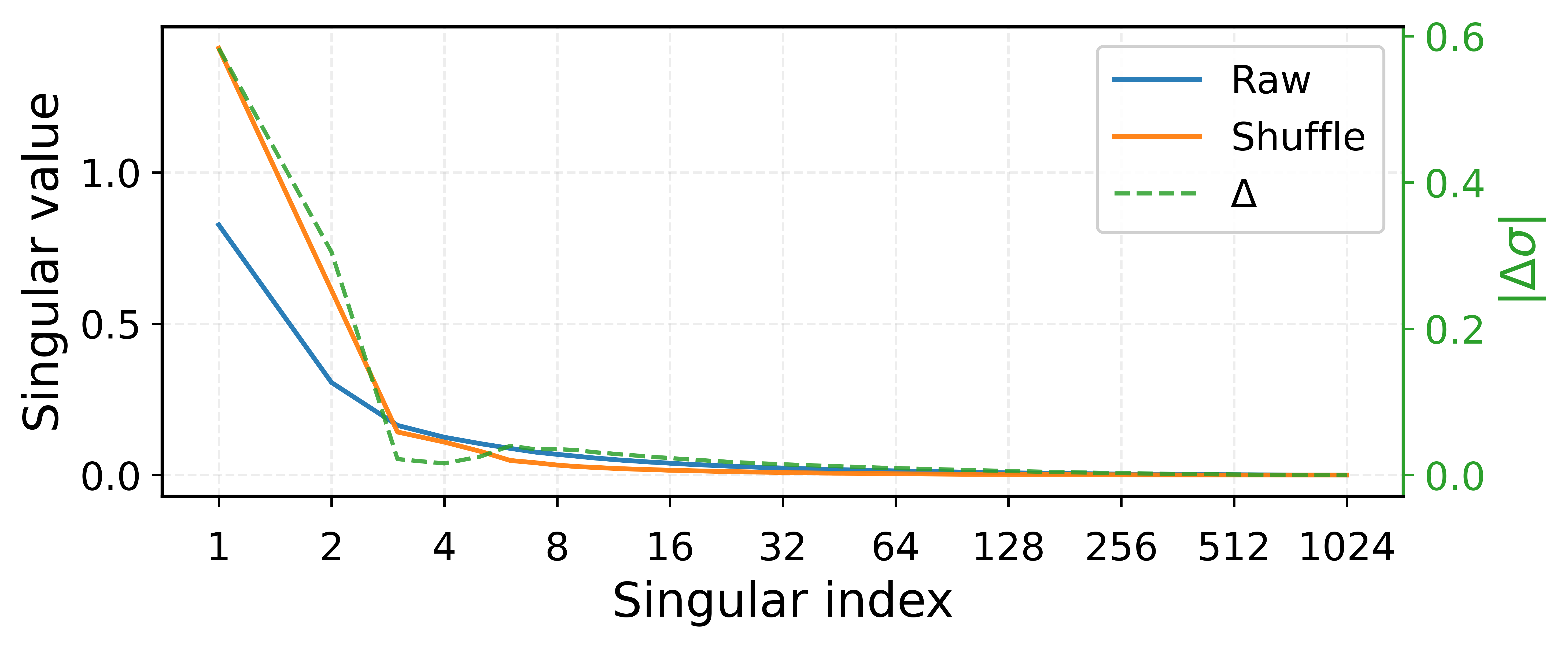}
    \caption{Gradient spectrum under two controlled interventions on Qwen3-0.6B: frequency-normalized loss (\emph{FreqNorm}, top) selectively suppresses the leading spike components, while intra-sentence token permutation (\emph{Shuffle}, bottom) selectively amplifies them; in both cases, changes rapidly vanish in the tail.}
    \label{fig:spike-semantic}
    \vspace{-1\baselineskip}
\end{figure}

Figure~\ref{fig:spike-semantic} reports results on Qwen3-0.6B, with LLaMA3-8B deferred to Appendix~\ref{appendix:semantic-correspondence}, and shows that both interventions perturb the spectrum almost exclusively in the spike: the absolute change $|\Delta\sigma_k|$ concentrates on the leading singular components and quickly vanishes in the tail.
Under \emph{FreqNorm}, the spike is selectively suppressed and the largest singular value drops by more than $25\%$, indicating that a substantial fraction of spike energy is driven by high-frequency-token contributions.
Under \emph{Shuffle}, the spike is selectively amplified, leading components increase by up to $75\%$, because disrupting word order removes syntactic structure that the pretrained model expects, inducing large corrective gradients concentrated in the spike subspace.

Together, these responses indicate that \emph{the spike predominantly reflects common grammatical signals driven by frequency skew and order-sensitive structure, while the smooth tail is comparatively robust and more associated with fine-grained, context-dependent semantics}.

\subsection{Spike Updating Suppresses Long-Tail Learning}
This subsection shows that spike dominance suppresses long-tail learning through two coupled mechanisms.
First, under AdamW-style optimization, spike-dominated second-moment accumulation controls element-wise normalization and contracts tail update magnitudes.
Second, spike-dominated stochastic gradient variance imposes a conservative learning-rate ceiling.

We measure the spectral structure of AdamW momentums on Qwen3-0.6B and LLaMA3-8B during pretraining.
We record the first moment $\mathbf{M}$ and the second moment $\mathbf{V}$, and compute their singular spectra together with the cumulative energy distribution (CDF),
$\mathrm{CDF}(j)=\sum_{i=1}^{j}\sigma_i^2\big/\sum_{i=1}^{r}\sigma_i^2$.
To isolate spike-dominated normalization, we decompose each moment into a spike projection and a residual tail:
$\mathbf{M}=\mathbf{M}_s+\mathbf{M}_t$ and $\mathbf{V}=\mathbf{V}_s+\mathbf{V}_t$,
where $\mathbf{M}_s \triangleq P_k(\mathbf{M})$ and $\mathbf{V}_s \triangleq P_k(\mathbf{V})$ denote the rank-$k$ truncated SVD reconstructions, and $\mathbf{M}_t\triangleq \mathbf{M}-\mathbf{M}_s$, $\mathbf{V}_t\triangleq \mathbf{V}-\mathbf{V}_s$.
Under the AdamW update $\Delta \mathbf{W}=-\eta\,\mathbf{M}/(\sqrt{\mathbf{V}}+\epsilon)$, the tail contribution is
$\Delta \mathbf{W}_t=-\eta\,\mathbf{M}_t/(\sqrt{\mathbf{V}_s+\mathbf{V}_t}+\epsilon)$,
showing that tail updates are normalized by a denominator dominated by $\mathbf{V}_s$ when $\mathbf{V}$ is highly anisotropic.
We visualize the suppression by comparing the element-wise magnitudes of
$\mathbf{M}_t/(\sqrt{\mathbf{V}_s+\mathbf{V}_t}+\epsilon)$
against the tail-only baseline
$\mathbf{M}_t/(\sqrt{\mathbf{V}_t}+\epsilon)$.

\begin{figure}[t]
    \centering
    \includegraphics[width=\linewidth]{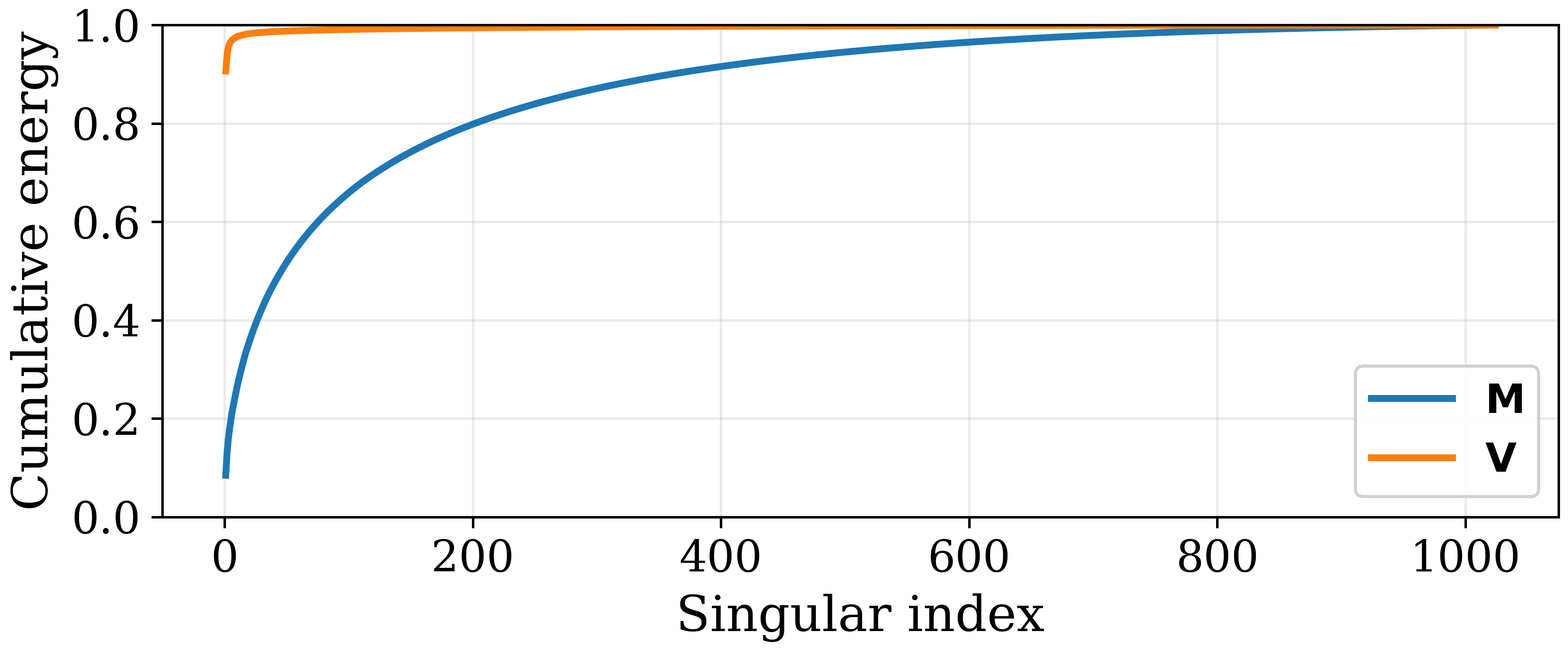}\\[0.5em]
    \includegraphics[width=\linewidth]{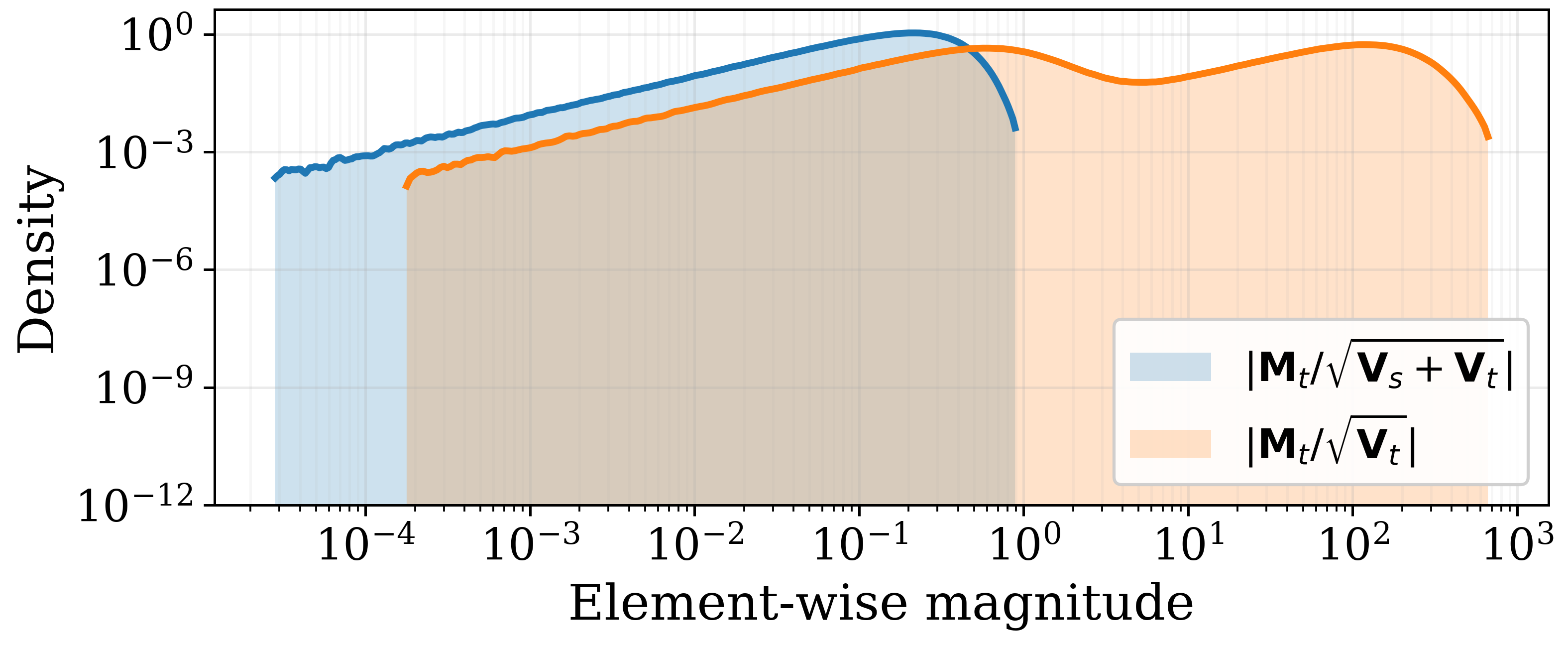}
    \caption{{Spike-dominated second-moment accumulation suppresses tail updates (Qwen3-0.6B).}
    \textit{Top:} cumulative spectral energy (CDF) of AdamW moments, showing that the second moment $\mathbf{V}$ is far more spike-concentrated than the first moment $\mathbf{M}$.
    \textit{Bottom:} element-wise magnitudes of tail updates, where full normalization $\mathbf{M}_t/(\sqrt{\mathbf{V}_s+\mathbf{V}_t}+\epsilon)$ is strongly contracted relative to the tail-only baseline $\mathbf{M}_t/(\sqrt{\mathbf{V}_t}+\epsilon)$.}
    \label{fig:spike-suppress-tail}
    \vspace{-0.8em}
\end{figure}

Figure~\ref{fig:spike-suppress-tail} (Qwen3-0.6B) shows that AdamW’s second moment $\mathbf{V}$ is much more spike-dominated than the first moment $\mathbf{M}$: the spike subspace already explains about $97\%$ of the spectral energy in $\mathbf{V}$, while accounting for only about $50\%$ in $\mathbf{M}$. Results on LLaMA3-8B are provided in Appendix~\ref{appendix:spike-suppress-tail-8b}.
This separation implies that element-wise normalization is effectively governed by spike-driven variance accumulation.
Accordingly, the tail-update distribution under full normalization, $\mathbf{M}_t/(\sqrt{\mathbf{V}_s+\mathbf{V}_t}+\epsilon)$, is markedly smaller than the tail-only baseline $\mathbf{M}_t/(\sqrt{\mathbf{V}_t}+\epsilon)$, with the latter shifted upward by roughly two orders of magnitude.
Overall, spike-dominated second-moment accumulation sharply contracts the effective step size available to long-tail directions.

Beyond element-wise suppression, spike-dominated stochastic variance also bounds the mean-optimal learning rate.
Consider the expected loss $L(\mathbf{w})$ for a layer weight matrix $\mathbf{W}\in\mathbb{R}^{m\times n}$ with $\mathbf{w}\triangleq \mathrm{vec}(\mathbf{W})$.
Let $\mathbf{G}\in\mathbb{R}^{m\times n}$ be the random mini-batch gradient computed with batch size $B$, with mean
$\bar{\mathbf{G}}\triangleq \mathbb{E}[\mathbf{G}]$.
Vectorizing gives $\mathbf{g}\triangleq \mathrm{vec}(\mathbf{G})$ and $\bar{\mathbf{g}}\triangleq \mathbb{E}[\mathbf{g}]=\mathrm{vec}(\bar{\mathbf{G}})$, and we write
$\mathrm{Cov}(\mathbf{g})=\mathbf{\Sigma}/B$ for some per-sample covariance $\mathbf{\Sigma}$.
Let $\mathbf{H}\triangleq \nabla^2 L(\mathbf{w})$ be the Hessian at the current iterate.
For the spike subspace, take the SVD of the mean gradient $\bar{\mathbf{G}}=\sum_{i=1}^{r}\sigma_i\,\mathbf{u}_i\mathbf{v}_i^\top$,
define $\mathbf{s}_i\triangleq \mathrm{vec}(\mathbf{u}_i\mathbf{v}_i^\top)$, and let
$\mathbf{\Pi}_k\triangleq \sum_{i=1}^{k}\mathbf{s}_i\mathbf{s}_i^\top$ be the projector onto $\mathrm{span}\{\mathbf{s}_1,\ldots,\mathbf{s}_k\}$.
We denote the spike-restricted covariance by $\mathbf{\Sigma}_s\triangleq \mathbf{\Pi}_k\,\mathbf{\Sigma}\,\mathbf{\Pi}_k$.
A second-order expansion of $L$ around $\mathbf{w}$ and taking expectation yields a quadratic surrogate in $\eta$, whose minimizer gives the mean-optimal learning rate. The full proof is provided in Appendix~\ref{appendix:proof_spike_lr_bound}.

\begin{theorem}[Spike-dominated variance bounds the mean-optimal learning rate]
\label{thm:spike_lr_bound}
Assume $\mathbf{H}\succeq \mathbf{0}$ and consider the update $\mathbf{w}^{+}=\mathbf{w}-\eta\,\mathbf{g}$.
Under the second-order approximation of $\mathbb{E}[L(\mathbf{w}^{+})]$, the mean-optimal learning rate is
\begin{equation}
\eta^\ast \;=\;
\frac{\|\bar{\mathbf{g}}\|_2^2}{\bar{\mathbf{g}}^\top \mathbf{H}\bar{\mathbf{g}}+\frac{1}{B}\,\mathrm{tr}(\mathbf{\Sigma}\mathbf{H})}.
\label{eq:eta_star}
\end{equation}
Moreover, it is upper bounded by the spike variance contribution:
\begin{equation}
\eta^\ast
\;\le\;
\frac{\|\bar{\mathbf{g}}\|_2^2}{\bar{\mathbf{g}}^\top \mathbf{H}\bar{\mathbf{g}}+\frac{1}{B}\,\mathrm{tr}(\mathbf{\Sigma}_s\mathbf{H})}
\;\le\;
\frac{B\,\|\bar{\mathbf{g}}\|_2^2}{\mathrm{tr}(\mathbf{\Sigma}_s\mathbf{H})}.
\label{eq:eta_star_spike_upper}
\end{equation}
If $\mathbf{H}\succeq \mu \mathbf{I}$ for some $\mu>0$, then
\begin{equation}
\eta^\ast \;\le\; \frac{B\,\|\bar{\mathbf{g}}\|_2^2}{\mu\,\mathrm{tr}(\mathbf{\Sigma}_s)}
\;=\;
\frac{B\,\|\bar{\mathbf{g}}\|_2^2}{\mu\sum_{i=1}^{k}\mathbf{s}_i^\top \mathbf{\Sigma}\mathbf{s}_i}.
\label{eq:eta_star_spike_upper_mu}
\end{equation}
\end{theorem}

Theorem~\ref{thm:spike_lr_bound} shows that the mean-optimal learning rate is governed by the curvature--variance denominator
$\bar{\mathbf{g}}^\top \mathbf{H}\bar{\mathbf{g}}+\mathrm{tr}(\mathbf{\Sigma}\mathbf{H})/B$.
When stochastic variance concentrates in the spike subspace, the effective bound tightens to the spike term
$\mathrm{tr}(\mathbf{\Sigma}_s\mathbf{H})/B$, so $\eta^\ast$ is primarily constrained by common-structure fluctuations.

Taken together, \emph{spike dominance suppresses long-tail learning through both local and global mechanisms:
it contracts tail updates via spike-driven second-moment normalization, and it caps the globally stable step size through spike-dominated stochastic variance, leaving tail directions to evolve under persistently underpowered updates}.

\subsection{Smaller Singular Directions Encode Sparser Semantics with Higher Relative Variance}
This subsection examines tail singular directions as carriers of sparse semantic signals, where only a small fraction of samples yield non-negligible projections onto a given direction. We assess their reliability under stochastic gradients using \emph{relative variance}, a scale-normalized measure of per-direction fluctuation.

We form a reference gradient $\bar{\mathbf{G}}$ by averaging gradients over a very large batch ($B_{\mathrm{ref}}=1024$) in Qwen3-0.6B, and compute its SVD
$\bar{\mathbf{G}}=\sum_{k=1}^{r}\sigma_k\,\mathbf{u}_k\mathbf{v}_k^\top$ to define a fixed spectral basis.
We then collect a large set of stochastic micro-batch gradients $\{\mathbf{G}_i\}$ and project each onto the $k$-th spectral component using $a_{i,k}\;\triangleq\;\mathbf{u}_k^\top \mathbf{G}_i \mathbf{v}_k .$
We then compute the per-direction relative variance $\mathrm{RelVar}(k)\;\triangleq\;\frac{\mathrm{Var}(a_{i,k})}{\sigma_k^2},$
which measures the noise sensitivity of each spectral direction relative to its signal magnitude.

\begin{figure}[ht]
    \centering
    \includegraphics[width=\linewidth]{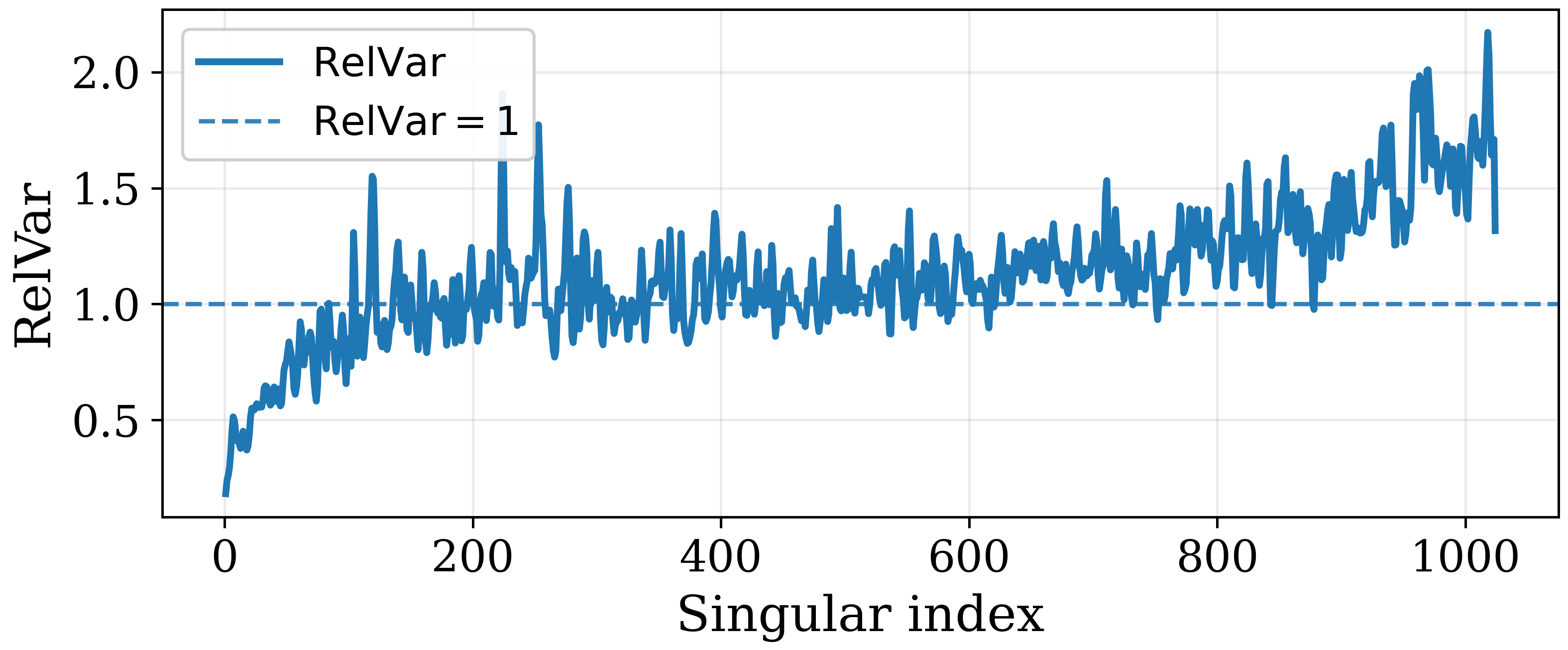}
    \caption{$\mathrm{RelVar}(k)=\mathrm{Var}(a_k)/\sigma_k^2$ increases with $k$, indicating more noise-dominated small-singular directions.}
    \label{fig:relvar}
\end{figure}

Figure~\ref{fig:relvar} shows a monotonic rise of $\mathrm{RelVar}(k)$ with the singular-value index $k$ across micro-batch settings, with the increase most pronounced in the tail.
Consequently, \emph{smaller-singular directions are increasingly noise-dominated}: their stochastic fluctuations are large relative to their signal scale, consistent with tail components encoding sparser and more intermittent semantics.

\subsection{Numerical Variance in Iterative Methods Disproportionately Rotates Tail Singular Directions}

This subsection analyzes \emph{numerical variance} introduced by iterative spectral routines.
Using Newton--Schulz (NS) iteration as a representative example, we show that while spike directions remain relatively stable, tail directions can be substantially rotated by iterative updates.

Let $\mathrm{NS}(\mathbf{G})$ denote the matrix produced by NS iteration, and let
$\{\mathbf{v}_i\}_{i=1}^{r}$ and $\{\widehat{\mathbf{v}}_i\}_{i=1}^{r}$ be the right singular vectors of $\mathbf{G}$ and $\mathrm{NS}(\mathbf{G})$, respectively.
We quantify per-direction preservation by $\mathrm{align}(i)\;\triangleq\;\max_{j\in[r]}\big|\langle \mathbf{v}_i,\widehat{\mathbf{v}}_j\rangle\big|,$
where a value close to $1$ indicates the $i$-th direction is preserved, and a small value indicates severe rotation.

\begin{figure}[ht]
    \centering
    \includegraphics[width=\linewidth]{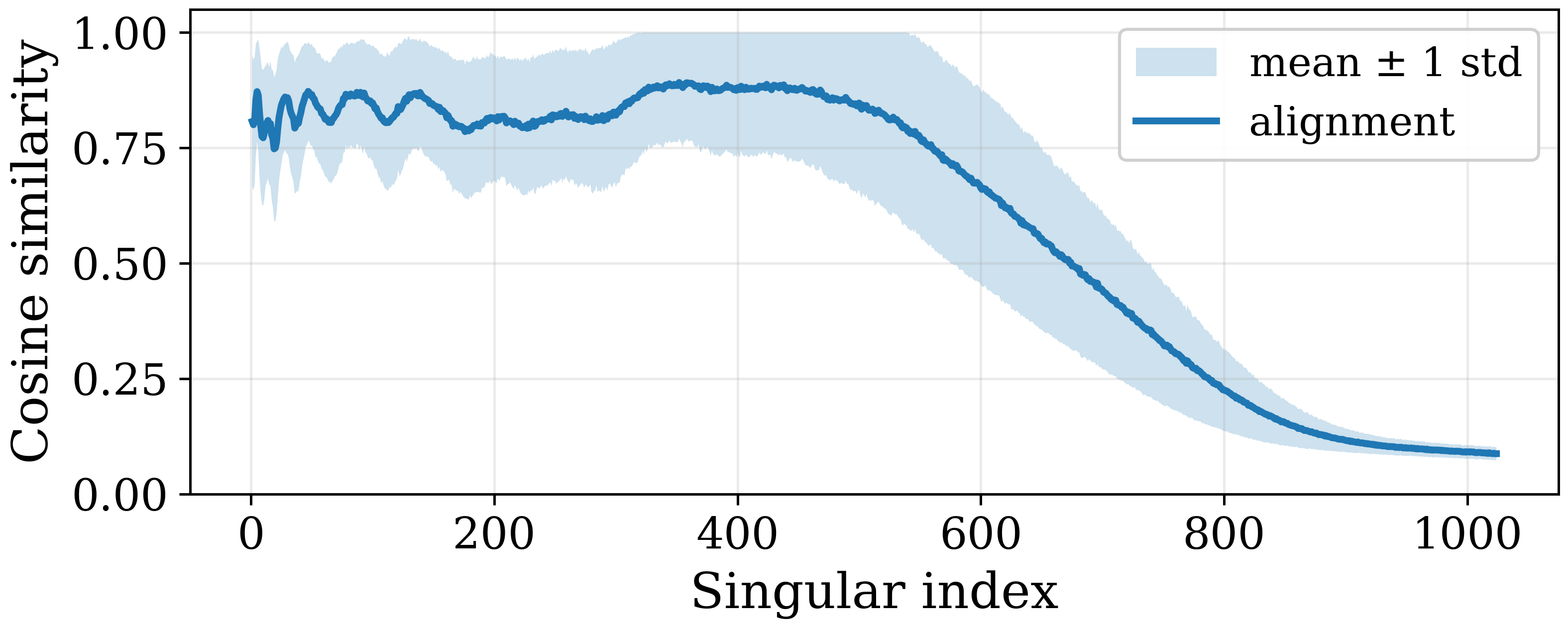}
    \caption{Alignment between singular directions of $G$ and $\mathrm{NS}(G)$.
    NS largely preserves head directions but severely disrupts tail directions.}
    \label{fig:ns_alignment}
    \vspace{-0.8em}
\end{figure}

Figure~\ref{fig:ns_alignment} shows a clear head--tail split: leading singular vectors remain well aligned, around 0.85, while alignment drops monotonically in the tail and approaches 0.1.

Therefore, iterative spectral processing is not numerically neutral for long-tail learning: \emph{it preferentially perturbs directions that already have weak and noise-sensitive signals, and more aggressive tail equalization would further exacerbate this effect while adding computation.}




\section{Method}
The findings in analysis section lead to a simple design principle:
\emph{attenuate the dominant common-feature spike subspace while avoiding aggressive amplification of the noise-dominated spectral tail.}

Motivated by this principle, we propose \textit{Spectra}, a spike-aware optimizer that explicitly operates only on a low-dimensional spike subspace and leaves the tail unamplified.
Spectra maintains a cached estimate of the spike subspace via warm-started power iteration, updates it intermittently, and performs spike singular-value shrinking towards the average scale of the tail.

\subsection{The Spectra Optimizer}
\label{sec:spectra_overview}

Spectra maintains a momentum matrix and performs spike singular-value shrinking on this momentum.
At each step, it (i) updates the momentum, (ii) estimates a rank-$k$ spike subspace via warm-started power iteration, (iii) replaces the spike singular values with a tail-scale estimate while keeping the tail residual unchanged, and (iv) normalizes the step size using the RMS scale of the resulting shaped update.
This suppresses spike-dominated updates without equalizing or amplifying the noise-sensitive tail, and avoids dense per-parameter second-moment statistics.

\begin{algorithm}[h]
\caption{Spectra Optimizer Step}
\label{alg:spectra_merged}
\begin{algorithmic}[1]
\STATE {\bfseries Input:} weights $\mathbf{W}_{t-1}\in\mathbb{R}^{m\times n}$, gradient $\mathbf{G}_t\in\mathbb{R}^{m\times n}$, momentum $\mathbf{M}_{t-1}\in\mathbb{R}^{m\times n}$
\STATE {\bfseries Hyperparams:} learning rate $\eta$, momentum coefficient $\mu$, rank ratio $r$, power iterations $T$, $\epsilon$
\STATE {\bfseries Output:} updated weights $\mathbf{W}_t$, updated momentum $\mathbf{M}_t$
\STATE $k \leftarrow \max\!\big(1,\mathrm{round}(r\cdot \min(m,n))\big)$
\STATE $\mathbf{M}_t \leftarrow \mu \mathbf{M}_{t-1} + \mathbf{G}_t$
\STATE $(\mathbf{U}_k, \mathbf{s}_k, \mathbf{V}_k) \leftarrow \textsc{PowerIterationSVD}(\mathbf{M}_t, k, T)$
\STATE $\mathbf{M}_{\mathrm{tail}} \leftarrow \mathbf{M}_t - \mathbf{U}_k\,\mathrm{diag}(\mathbf{s}_k)\,\mathbf{V}_k^\top$
\STATE $\sigma_{\mathrm{tail}} \leftarrow \sqrt{\|\mathbf{M}_{\mathrm{tail}}\|_F^2 / (\min(m,n)-k)}$
\STATE $\mathbf{O}_t \leftarrow \mathbf{M}_{\mathrm{tail}} + \mathbf{U}_k\,\mathrm{diag}(\sigma_{\mathrm{tail}}\mathbb{I})\,\mathbf{V}_k^\top$
\STATE $RMS \leftarrow \|\mathbf{O}_t\|_F / \sqrt{mn}$
\STATE $\eta' \leftarrow 0.2\,\eta / (RMS+\epsilon)$
\STATE $\mathbf{W}_t \leftarrow \mathbf{W}_{t-1} - \eta' \mathbf{O}_t$
\STATE {\bfseries return} $(\mathbf{W}_t, \mathbf{M}_t)$
\end{algorithmic}
\end{algorithm}

\noindent
Unlike full spectrum-flattening, e.g. Muon, which enforces global equalization and can over-emphasize small-singular, noise-dominated modes, Spectra performs a \emph{localized} intervention: it keeps the residual tail $\mathbf{M}_{\mathrm{tail}}$ unchanged and only shrinks the spike singular values toward the tail’s average scale $\sigma_{\mathrm{tail}}$.
The final update $\mathbf{O}_t$ is then used both for the parameter update and for RMS normalization, ensuring that step-size calibration reflects the shaped update actually applied.

\subsection{Efficient Spike Subspace Estimation}
\label{sec:subspace}

A practical obstacle for spectral-domain optimization is the cost of repeatedly computing SVDs on large matrices.
For $\mathbf{G}\in\mathbb{R}^{m\times n}$, a full SVD scales as $\mathcal{O}(\min(m,n)\,mn)$ and is infeasible in the training inner loop.
However, gradient energy concentrates in a compact low-rank spike.
Spectra therefore only needs a \emph{rank-$k$} approximation that tracks this dominant subspace, rather than an exact factorization of the full spectrum.

\textbf{Cached subspace iteration.}
We estimate the spike subspace using a cached power-iteration routine (Algorithm~\ref{alg:power_iter_svd}).
The key is to exploit temporal continuity: the spike subspace changes slowly across steps, so the previous right-singular subspace provides an accurate warm start.
Concretely, we cache the rank-$k$ right subspace $\mathbf{V}$ from the previous step and initialize the current iteration with it.
Each iteration applies two inexpensive projections, $\mathbf{P}\leftarrow \mathbf{G}\mathbf{V}$ and $\mathbf{W}\leftarrow \mathbf{G}^\top \mathbf{U}$, interleaved with a thin QR orthonormalization to maintain numerical stability.
After $T$ iterations, we output a rank-$k$ estimate $(\mathbf{U}_k,\mathbf{s}_k,\mathbf{V}_k)$ and refresh the cache with $\mathbf{V}_k$.
In practice, warm-starting substantially reduces the iterations required to reliably track the spike subspace compared to cold-start randomized SVD.

\begin{algorithm}[t]
  \caption{Cached Power-Iteration SVD (rank-$k$)}
  \label{alg:power_iter_svd}
  \begin{algorithmic}[1]
    \STATE {\bfseries Input:} matrix $\mathbf{G}\in\mathbb{R}^{m\times n}$, rank $k$, iteration count $T$, cache state
    \STATE {\bfseries Output:} $(\mathbf{U}_k,\mathbf{s}_k,\mathbf{V}_k)$
    \STATE $\mathbf{V}^{(0)} \leftarrow \mathrm{State}[V_{\mathrm{cache}}]$
    \IF{$\mathbf{V}^{(0)}$ is \textbf{None}}
        \STATE $(\mathbf{U}_k,\mathbf{s}_k,\mathbf{V}_k) \leftarrow \textsc{svd\_lowrank}(\mathbf{G},k)$  \textit{// bootstrap}
        \STATE $\mathrm{State}[V_{\mathrm{cache}}]\leftarrow \mathbf{V}_k$; \ \textbf{return} $(\mathbf{U}_k,\mathbf{s}_k,\mathbf{V}_k)$
    \ENDIF
    \FOR{$i=1$ {\bfseries to} $T$}
        \STATE $\mathbf{P} \leftarrow \mathbf{G} \mathbf{V}^{(i-1)}$
        \STATE $\mathbf{U}^{(i)} \leftarrow \textsc{ThinQR}(\mathbf{P})$
        \STATE $\mathbf{W} \leftarrow \mathbf{G}^\top \mathbf{U}^{(i)}$
        \STATE $\mathbf{s}^{(i)} \leftarrow \textsc{ColNorms}(\mathbf{W})$
        \STATE $\mathbf{V}^{(i)} \leftarrow \mathbf{W}\,\mathrm{diag}\!\big((\mathbf{s}^{(i)})^{-1}\big)$ \hfill \textit{// normalize columns}
    \ENDFOR
    \STATE $\mathbf{U}_k \leftarrow \mathbf{U}^{(T)}$, $\mathbf{s}_k \leftarrow \mathbf{s}^{(T)}$, $\mathbf{V}_k \leftarrow \mathbf{V}^{(T)}$
    \STATE $\mathrm{State}[V_{\mathrm{cache}}]\leftarrow \mathbf{V}_k$ \hfill \textit{// update cache}
    \STATE \textbf{return} $(\mathbf{U}_k,\mathbf{s}_k,\mathbf{V}_k)$
  \end{algorithmic}
\end{algorithm}

\begin{figure}[ht]
    \centering
    \includegraphics[width=\linewidth]{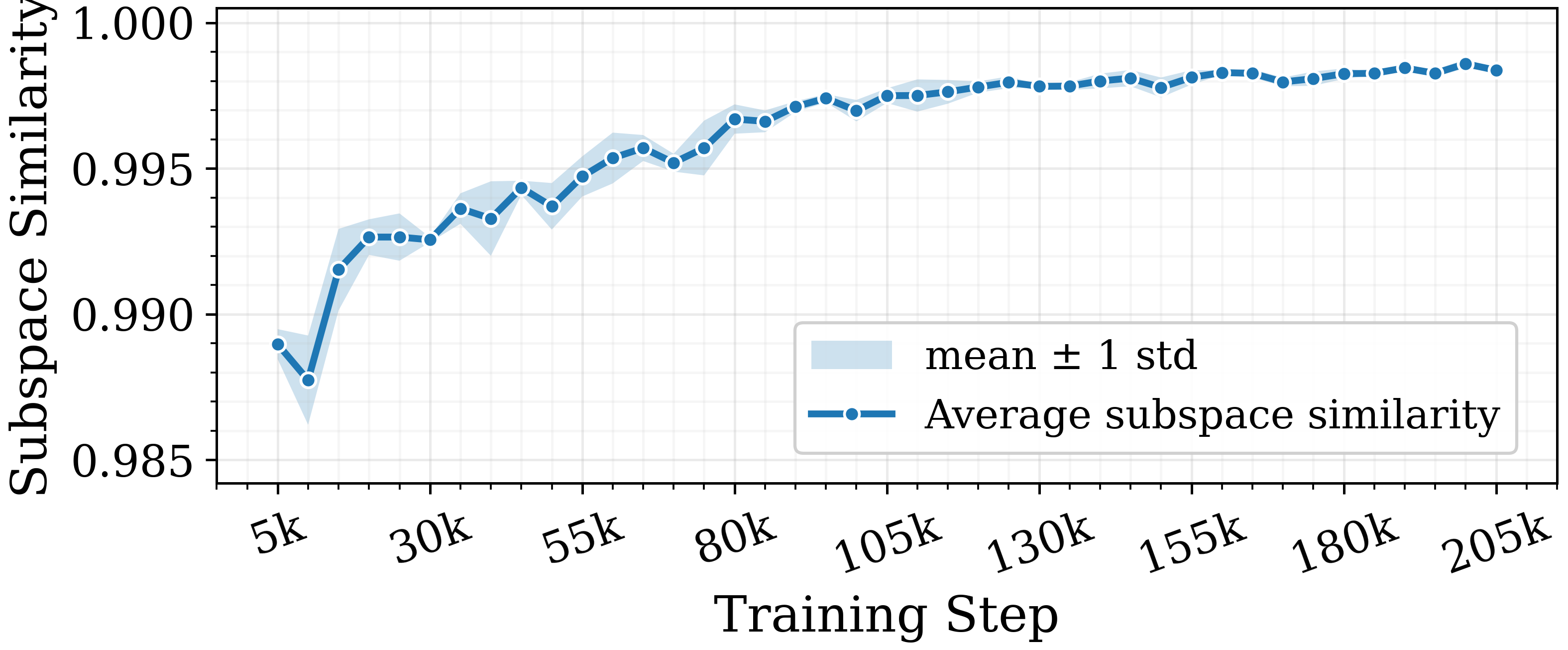}
    \caption{Temporal continuity of the spike subspace. We report the step-to-step similarity between the top-$k$ right-singular subspaces of consecutive gradients, showing consistently high similarity.}
    \label{fig:subspace_similarity}
\end{figure}

\textbf{Empirical justification: subspace continuity.}
Caching is effective only if the spike subspace is stable across adjacent steps.
We therefore measure the similarity between the spike right-singular subspaces of consecutive gradients using the largest canonical correlation.
Figure~\ref{fig:subspace_similarity} shows consistently high similarity (above $0.98$) throughout training, indicating that the spike subspace evolves slowly.
This stability justifies cached warm-starts and enables accurate tracking with only one power-iteration step, amortizing the cost of subspace estimation in Spectra.

\subsection{Efficiency Analysis}
\label{sec:complexity}

\textbf{Theoretical complexity.}
Table~\ref{tab:complexity} compares Spectra with AdamW~\cite{loshchilov2017decoupled} and Muon.
AdamW applies element-wise moment updates and thus costs $\mathcal{O}(mn)$ per step, but stores two dense moment buffers ($2mn$).
Muon performs Newton--Schulz iterations with full matrix multiplications, incurring $\mathcal{O}(T\cdot mn\min(m,n))$ time for $T$ iterations.
Spectra estimates only a rank-$k$ spike subspace via cached power iteration: each iteration is dominated by two matrix multiplications $\mathbf{G}\mathbf{V}$ and $\mathbf{G}^\top \mathbf{U}$, costing $\mathcal{O}(mnk)$, with an additional $\mathcal{O}(mk^2)$ thin-QR negligible cost when $k$ is a small fraction of the layer dimension.
Thus, Spectra’s overhead is $\mathcal{O}(Tmnk)$ for $T$ power-iteration steps, with $k\approx 0.015\min(m,n)$ in our default setting.

\begin{table}[ht]
  \caption{Optimizer states memory cost and per-step complexity for $\mathbf{G} \in \mathbb{R}^{m \times n}$ ($k \ll \min(m, n)$). $T$ denote the iteration counts of NS (Muon) and power iteration (Spectra), respectively.}
  \label{tab:complexity}
  \centering
  \small
  \begin{tabular}{lcc}
    \toprule
    Optimizer  & Memory Cost & Complexity \\
    \midrule
    AdamW & $2mn$ & $\mathcal{O}(mn)$  \\
    Muon  & $mn$  & $\mathcal{O}(Tmn\min(m,n))$ \\
    Spectra & $mn+nk$ & $\mathcal{O}(Tmnk)$  \\
    \bottomrule
  \end{tabular}
\end{table}

\textbf{Empirical Latency.} To validate the efficiency gains, we benchmark the latency of NS iterations against our Spectral Power Iteration on an NVIDIA H200 GPU. As shown in Table \ref{tab:latency}, for a standard LLM layer size ($4096 \times 14336$) in 8B model, Power-Iter with 1-2 iterations is $3.5\times$ to $5\times$ faster than NS. Even with 4 iterations, Spectra maintains a significant speed advantage. This efficiency allows Spectra to perform spectral preconditioning with negligible overhead compared to the backward pass.

\begin{table*}[h]
\centering
\caption{Empirical latency comparison (ms) on H200 GPU. Power-Iter uses a rank ratio of $1.5\%$. $T$ denotes the number of power iterations.}
\label{tab:latency}
\small
\begin{tabular}{l|c|cccc}
\toprule
\textbf{Matrix Size} & \textbf{NS} & \textbf{Power-Iter ($T=1$)} & \textbf{Power-Iter ($T=2$)} & \textbf{Power-Iter ($T=4$)} & \textbf{Power-Iter ($T=8$)} \\
\midrule
(4096, 4096)  & 3.5664 & 0.9724 & 1.6012 & 2.8648 & 5.3927 \\
(4096, 14336) & 9.1465 & 1.7799 & 2.5622 & 4.1315 & 7.2676 \\
\bottomrule
\end{tabular}
\end{table*}

\section{Experiments}

\textbf{Models and Datasets.} We conduct experiments on two architectures: Qwen3-0.6B on 100B tokens and LLaMA3-8B on 50B tokens. For pretraining, we use the DCLM~\cite{li2024datacomp} dataset. For downstream evaluation, we consider three task types: question answering (ARC~\cite{clark2018think}, RACE~\cite{lai2017race}, BoolQ~\cite{clark2019boolq}), classification (HellaSwag~\cite{zellers2019hellaswag}, PIQA~\cite{bisk2020piqa}), and cloze prediction (LAMBADA~\cite{kazemi2023lambada}). See Appendix~\ref{appendix:experiment-details} for settings.

\textbf{Baselines.} We compare Spectra against AdamW and Muon, which applies iterative Newton--Schulz updates to approximate orthogonalized matrix steps; for Qwen3-0.6B, we additionally include Dion~\cite{ahn2025dion}, which uses power iteration to form a low-rank update.

\subsection{Main Results}
\textbf{Training Loss and Convergence.} Figure~\ref{fig:loss_curve} illustrates the training loss curves for both the 0.6B and 8B models. Spectra demonstrates superior convergence efficiency across both scales: on the 0.6B model, Spectra achieves a final validation loss that is 2.1\% lower than AdamW and 1.4\% lower than Muon, and reaches a matched loss level 30\% faster in wall-clock time than AdamW. This scaling advantage is further confirmed on the 8B model, where Spectra outperforms AdamW and Muon by 1.5\% and 0.4\% in final loss, respectively.

\begin{table}[ht]
    \centering
    \caption{Computational efficiency comparison on Qwen3-0.6B. Step time is measured as seconds per step (ms/step) using a per-GPU batch size of 500K tokens. VRAM indicates peak memory usage per GPU.}
    \label{tab:efficiency}
    \small
    \begin{tabularx}{0.45\textwidth}{L c c}
    \toprule
    Optimizer & Step Time (ms) & VRAM Usage (GB)\\
    \midrule
    AdamW      & 5297 & 107.4\\
    Dion      & 5270 & 132.2\\
    Muon      & 5385 & 105.8\\
    Spectra   & \textbf{5259} & \textbf{103.5} \\
    \bottomrule
    \end{tabularx}
\end{table}

\begin{figure}[ht]
    \centering
    \begin{minipage}{0.48\textwidth}
        \centering
        \includegraphics[width=\textwidth]{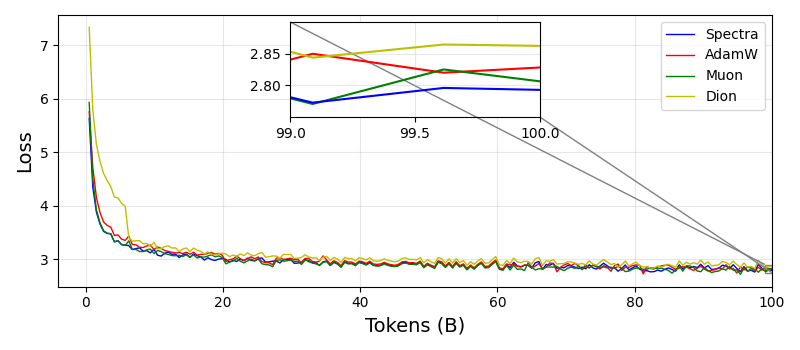}
        \centerline{(A) Qwen3-0.6B}
    \end{minipage}
    \hfill
    \begin{minipage}{0.48\textwidth}
        \centering
        \includegraphics[width=\textwidth]{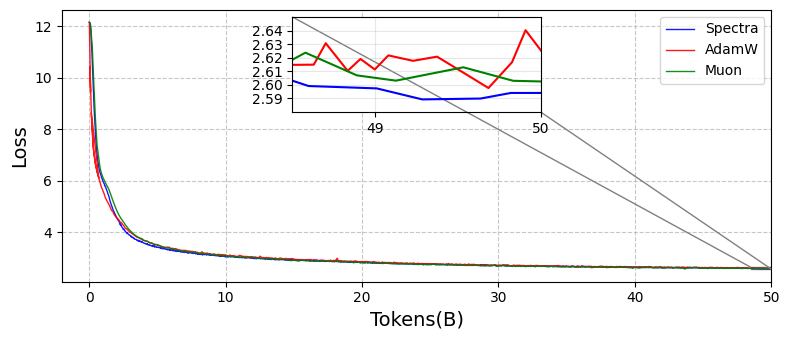} 
        \centerline{(B) LLaMA3-8B}
    \end{minipage}
    \caption{Training loss curves for (A) Qwen3-0.6B on 100B tokens and (B) LLaMA3-8B on 50B tokens.}
    \label{fig:loss_curve}
\end{figure}

\textbf{Downstream Performance.} Table~\ref{tab:Downstream_Combined} summarizes the performance across all benchmarks. Spectra consistently outperforms both AdamW and Muon at both scales. 

On Qwen3-0.6B, Spectra improves average accuracy by +1.41\% over AdamW and +0.89\% over Muon. On LLaMA3-8B, Spectra also achieves the best average accuracy, improving over AdamW and Muon by +1.62\% and +0.66\%, respectively. Together, these gains persist as we scale from 0.6B to 8B.

\textbf{Computational Efficiency.} We compare Spectra with AdamW, Muon, and Dion~\cite{ahn2025dion} on Qwen3-0.6B using NVIDIA H200 GPUs. Table~\ref{tab:efficiency} reports the step time and peak memory. Relative to AdamW, Spectra is 0.7\% faster in per-step wall-clock time, reduces optimizer state memory by 49.25\%, and lowers peak end-to-end memory by 3.63\%. Relative to Muon, Spectra is $5.1\times$ faster in optimizer processing time per step and reduces peak end-to-end memory by 2.17\%. Relative to Dion, Spectra achieves comparable step time (0.21\% faster) while reducing peak end-to-end memory by 21.7\%.

\begin{table*}[t]
  \caption{Downstream performance comparison for Qwen3-0.6B trained on 100B and LLaMA3-8B trained on 50B tokens. Spectra consistently achieves the highest average accuracy across both model scales.}
  \label{tab:Downstream_Combined}
  \centering
  \small
  \setlength{\tabcolsep}{4pt} 
  \begin{tabular*}{\textwidth}{@{\extracolsep{\fill}} llccccccccc @{}}
    \toprule
    Model & Optimizer & Loss & ArcC & ArcE & BoolQ & HellaSwag & LAMBADA & PIQA & RACE & Avg \\
    \midrule

    \multirow{4}{*}{Qwen3-0.6B}
      & AdamW   & 2.83 & \textbf{29.95} & 54.34 & 57.00 & 51.21 & 46.87 & 72.09 & 49.41 & 51.55 \\
      & Dion    & 2.86 & 27.82 & 53.41 & 53.12 & 49.25 & 44.75 & 70.84 & 49.41 & 49.80 \\
      & Muon    & 2.81 & 29.61 & 54.50 & 56.91 & \textbf{52.80} & \textbf{48.81} & 72.09 & 49.80 & 52.07 \\
      & Spectra & \textbf{2.77} & 29.61 & \textbf{56.61} & \textbf{60.21} & 52.30 & 48.44 & \textbf{73.18} & \textbf{50.34} & \textbf{52.96} \\
    \midrule

    \multirow{3}{*}{LLaMA3-8B}
      & AdamW   & 2.63 & 38.40 & 65.11 & 57.23 & 61.29 & 48.27 & 75.21 & 49.21 & 56.39 \\
      & Muon    & 2.60 & \textbf{38.77} & 65.94 & 59.84 & 61.43 & 47.97 & 75.13 & \textbf{52.36} & 57.35 \\
      & Spectra & \textbf{2.59} & 38.46 & \textbf{67.11} & \textbf{60.08} & \textbf{62.95} & \textbf{49.17} & \textbf{76.43} & 51.88 & \textbf{58.01} \\
    \bottomrule
  \end{tabular*}
\end{table*}

\begin{table*}[t]
  \caption{Ablation studies on Qwen3-0.6B. We explore different rank ratios ($r$) and power iteration counts ($T$). All models are trained on 50B tokens.}
  \label{tab:ablation}
  \centering
  \small
  \setlength{\tabcolsep}{4pt}
  \renewcommand{\arraystretch}{1.08}

  \begin{tabular*}{\textwidth}{@{\extracolsep{\fill}} ll *{8}{c} @{}}
    \toprule
    Ablation & Setting & ArcC & ArcE & BoolQ & HellaSwag & LAMBADA & PIQA & RACE & Avg \\
    \midrule

    \multirow{4}{*}{Rank Ratio $r$}
      & 1.5\% (Default) & \textbf{28.92} & 51.52 & 58.93 & 48.52 & \textbf{46.34} & 71.33 & 49.49 & 50.72 \\
      & 5\%             & 26.37 & 52.40 & \textbf{61.31} & 48.89 & 45.22 & 70.95 & 49.57 & 50.67 \\
      & 10\%            & 28.24 & 53.07 & 59.60 & 49.12 & 45.72 & 70.73 & 49.33 & 50.83 \\
      & 15\%            & 27.99 & \textbf{53.45} & 57.83 & \textbf{49.45} & 46.28 & \textbf{71.55} & \textbf{49.86} & \textbf{50.92} \\
    \midrule

    \multirow{4}{*}{Power Iter $T$}
      & $T=1$ (Default) & \textbf{28.92} & 51.52 & \textbf{58.93} & 48.52 & \textbf{46.34} & \textbf{71.33} & 49.49 & \textbf{50.72} \\
      & $T=2$           & 27.73 & \textbf{52.27} & 56.15 & \textbf{48.96} & 45.20 & 70.35 & \textbf{50.18} & 50.12 \\
      & $T=4$           & 27.47 & 52.19 & 53.82 & 48.76 & 45.39 & 70.62 & 50.04 & 49.76 \\
      & $T=8$           & 28.41 & 52.15 & 50.52 & 47.75 & 45.74 & 69.59 & 50.00 & 49.17 \\
    \bottomrule
  \end{tabular*}
\end{table*}

\subsection{Ablation Study}
\label{sec:ablation}

To further quantify the factors contributing to the effectiveness of Spectra, we conduct a series of ablation studies on the Qwen3-0.6B model trained with 50B tokens. We investigate the impact of the rank ratio for head compression, the number of power iterations, and the optimizer's sensitivity to learning rates. The results across downstream tasks are summarized in Table~\ref{tab:ablation} and Figure~\ref{fig:lr_comparison}.

\textbf{Sensitivity to Rank Ratio ($r$).} 
We vary the rank ratio $r\in\{1.5\%,5\%,10\%,15\%\}$ used for spike compression. As shown in Table~\ref{tab:ablation} and Figure~\ref{fig:lr_comparison}, downstream performance is largely insensitive to $r$: the average accuracy changes by less than $0.25$ points across the full range, and individual tasks exhibit only small fluctuations. This suggests that even a minimal rank ratio ($r=1.5\%$) already captures the dominant spike directions, and increasing $r$ provides limited practical benefit. We therefore use $r=1.5\%$ by default to minimize overhead.

\textbf{Number of Power Iterations ($T$).} 
We ablate the number of power-iteration steps $T \in \{1,2,4,8\}$ used for subspace estimation. As shown in Table~\ref{tab:ablation}, increasing $T$ does not improve the overall downstream average and can slightly degrade performance, with the drop largely driven by BoolQ; other tasks exhibit only minor variations. Since $T=1$ also has the lowest compute overhead, we adopt a single cached power-iteration step by default.

\begin{figure}[ht]
    \centering
    \begin{minipage}{0.48\textwidth}
        \centering
        \includegraphics[width=\textwidth]{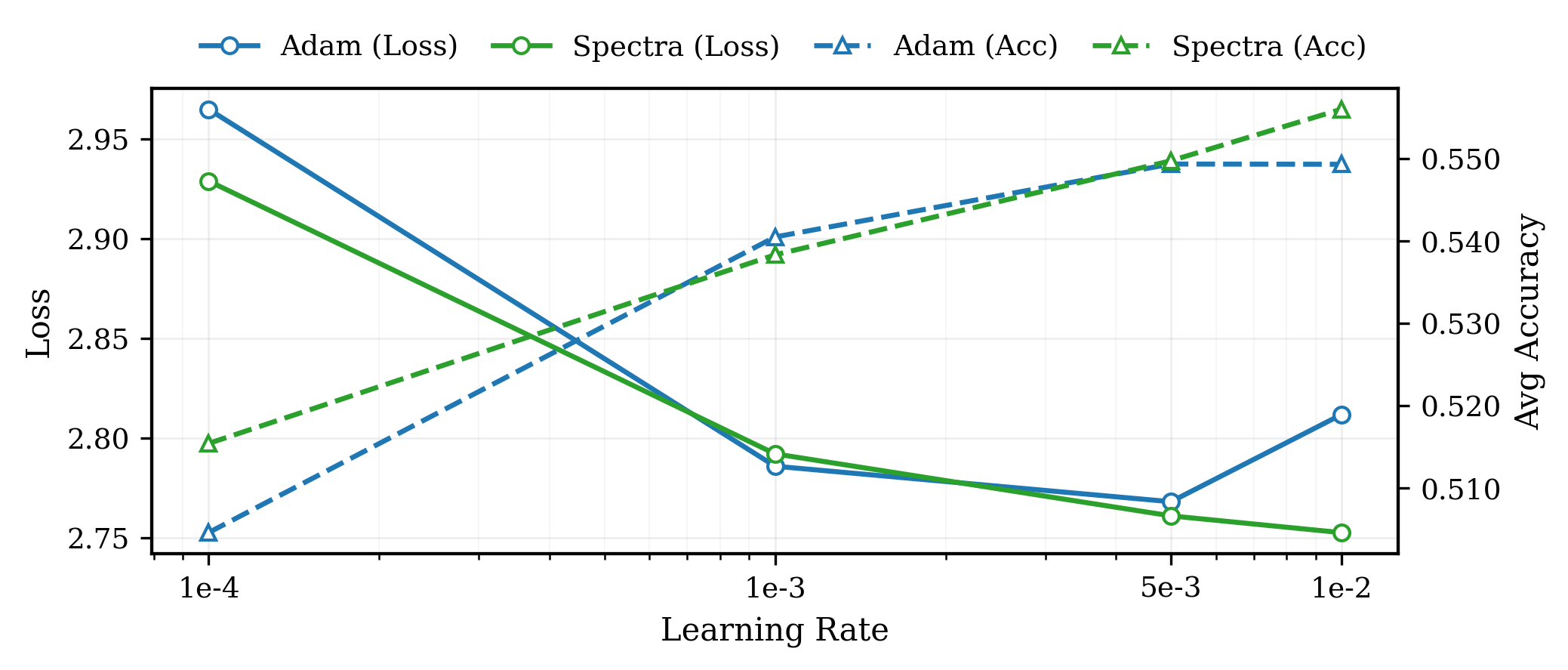} 
    \end{minipage}
    \caption{Comparison of Spectra and AdamW across learning rates $\eta \in \{8\times10^{-4}, 1\times10^{-3}, 5\times10^{-3}, 1\times10^{-2}\}$. Spectra shows superior convergence loss and downstream performance in most regimes.}
    \label{fig:lr_comparison}
\end{figure}

\textbf{Learning Rate Robustness.} 
We sweep the learning rate $\eta\in\{8\times10^{-4},\,1\times10^{-3},\,5\times10^{-3},\,1\times10^{-2}\}$ and compare against AdamW. Figure~\ref{fig:lr_comparison} shows that Spectra remains stable across this wide range and achieves better loss/accuracy at most settings. At $\eta=1\times10^{-2}$, where AdamW exhibits signs of instability, Spectra maintains robust convergence, indicating improved tolerance to larger step sizes.

\draft{
\section{Related Work}

\textbf{Element-wise Optimization.} AdamW~\cite{loshchilov2017decoupled} and related adaptive methods maintain per-parameter moment estimates and apply coordinate-wise rescaling. These methods operator  at the level of individual scalars and therefore does not explicitly exploit correlations within structured parameter blocks (e.g., weight matrices).

\textbf{Matrix-aware Preconditioning.} A line of work introduces matrix-structured preconditioners to capture cross-coordinate correlations. Shampoo~\cite{gupta2018shampoo} and SOAP~\cite{vyas2024soap} approximate second-order information using Kronecker-factored statistics, but their additional memory and compute can be substantial, typically scaling with $\mathcal{O}(m^2+n^2)$ for a matrix of shape $m\times n$.

\textbf{Orthogonal Update Methods.} Another approach performs updates through spectral transforms that normalize or orthogonalize matrix updates. Muon~\cite{jordan6muon} and Dion~\cite{ahn2025dion} use iterative procedures to obtain orthogonalized updates. PolarGrad~\cite{lau2025polargrad} further argues that the key advantage of such methods is to mitigate gradient anisotropy, but its reliance on full polar decomposition makes it expensive for large-scale pretraining.
}

\section{Related Work}

\textbf{Element-wise adaptive optimization.}
Adaptive methods such as AdamW \cite{loshchilov2017decoupled} rescale updates using per-parameter moment statistics and are widely used for LLM training. However, operating at the coordinate level, they ignore structured correlations within matrix-valued parameters. When gradient energy concentrates in a low-rank subspace, element-wise normalization becomes dominated by these directions, suppressing long-tail updates. Spectra differs by explicitly operating in the spectral domain and directly targeting this low-rank anisotropy.

\textbf{Matrix-aware preconditioning.}
Matrix-structured optimizers, including Shampoo \cite{gupta2018shampoo} and SOAP \cite{vyas2024soap}, capture cross-coordinate correlations via Kronecker-factored or second-order statistics. While effective, their memory and computational costs scale poorly with layer dimensions, limiting practicality for large-scale LLM pretraining. Spectra avoids approximating full second-order structure by exploiting the empirical low-rank concentration of gradients, enabling efficient low-rank spectral shaping with minimal overhead.

\textbf{Orthogonal and spectral update methods.}
Recent methods such as Muon \cite{jordan6muon}, Dion \cite{ahn2025dion}, and PolarGrad \cite{lau2025polargrad} reduce gradient anisotropy through orthogonalization or spectrum flattening. These approaches typically apply global spectral transformations, which can amplify noise-dominated small-singular directions and incur substantial numerical cost. In contrast, Spectra performs a localized intervention  suppressing the dominant spike subspace without amplifying the spectral tail.


\section{Conclusion}
\label{sec:conclusion}


We identify a persistent spike–tail structure in LLM gradients, where a small low-rank subspace dominates optimization and suppresses learning in the long tail. To address this asymmetry, we propose Spectra, a spike-aware optimizer that selectively suppresses dominant spectral components without amplifying noise-sensitive tail directions. By reshaping the gradient spectrum in this targeted manner, Spectra yields more benign conditioning in practice, enabling larger stable learning rates, faster convergence, and improved downstream performance with minimal computational and memory overhead. 
These results suggest that viewing gradients as structured spectral objects, rather than independent coordinates, offers a principled basis for scalable and robust LLM optimization, and that selectively targeting dominant structure can improve both stability and efficiency in practice.

\newpage
\section*{Impact Statement}
This paper presents work whose goal is to advance the field of machine learning. There are many potential societal consequences of our work, none of which we feel must be specifically highlighted here.



\newpage
\appendix
\onecolumn

\section{Appendix.}

\subsection{Gradient Anisotropy}
\label{appendix:gradient-anisotropy}
To complement Figure~\ref{fig:gradient-anisotropy} (deepest MLP), we report the gradient singular spectrum for three additional parameter matrices: (i) the shallowest attention $k$-projection, (ii) the shallowest MLP up-projection, and (iii) the deepest attention $k$-projection. Each plot overlays spectra at initialization and at convergence, and includes four Qwen3 model scales.

\begin{figure*}[ht]
    \centering
    \includegraphics[width=0.98\textwidth]{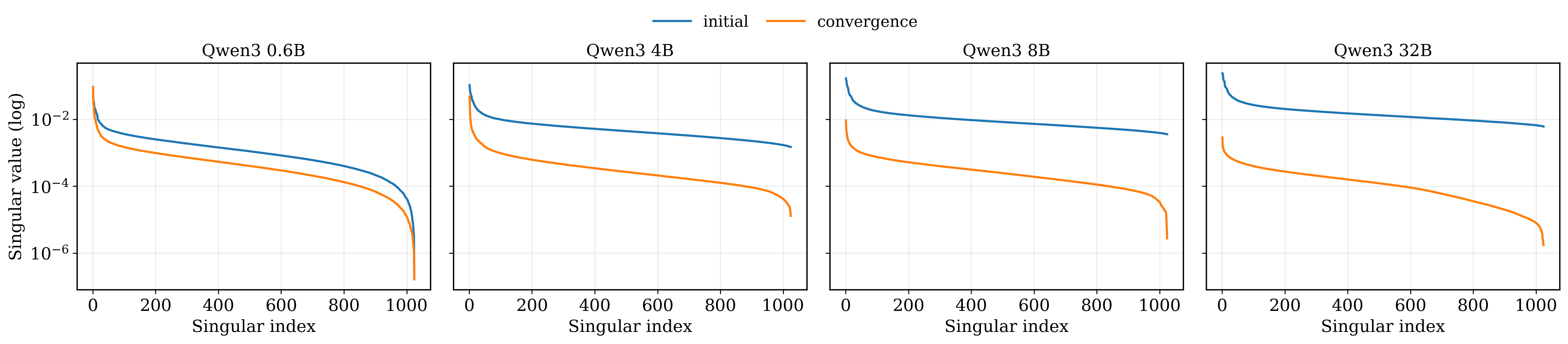}
    \caption{Gradient spectrum (initial vs. convergence) for the shallowest attention layer (self-attention $k$-projection) across four Qwen3 model scales.}
    \label{fig:grad_aniso_appendix_shallow_attn}
\end{figure*}

\begin{figure*}[ht]
    \centering
    \includegraphics[width=0.98\textwidth]{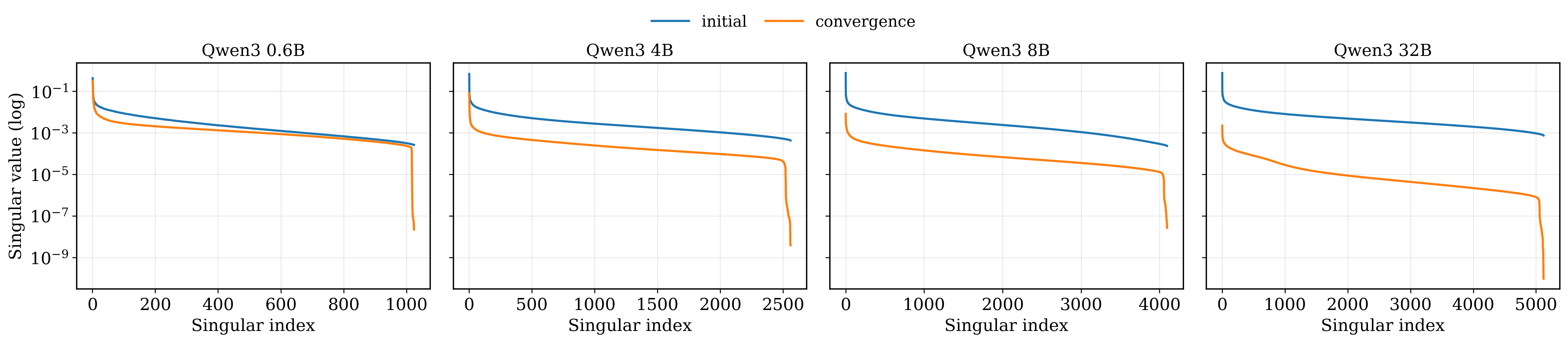}
    \caption{Gradient spectrum (initial vs. convergence) for the shallowest MLP layer (up-projection) across four Qwen3 model scales.}
    \label{fig:grad_aniso_appendix_shallow_mlp}
\end{figure*}

\begin{figure*}[ht]
    \centering
    \includegraphics[width=0.98\textwidth]{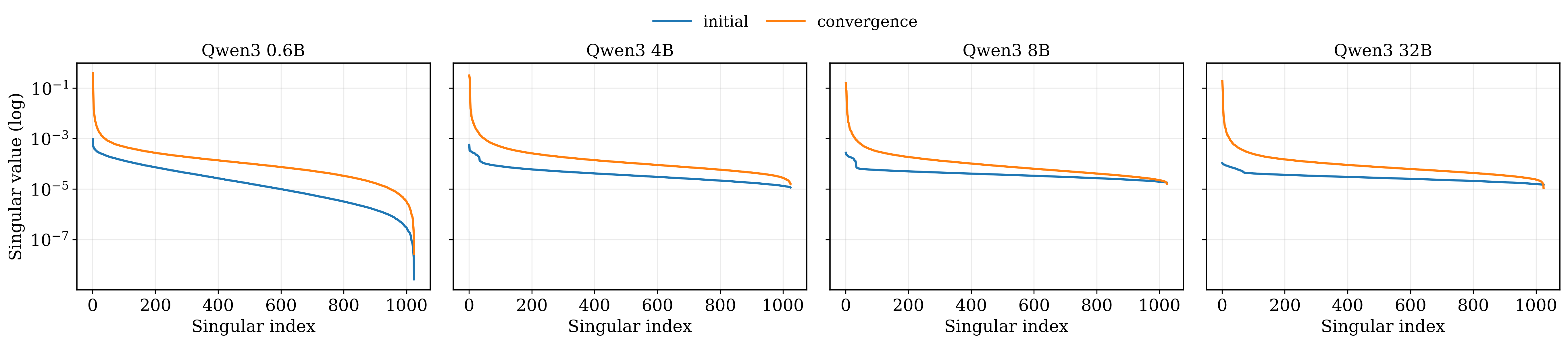}
    \caption{Gradient spectrum (initial vs. convergence) for a deeper attention layer (self-attention $k$-projection) across four Qwen3 model scales.}
    \label{fig:grad_aniso_appendix_deep_attn}
\end{figure*}

\subsection{Semantic Correspondence of Gradient Anisotropy}
\label{appendix:semantic-correspondence}

Figure~\ref{fig:spike-semantic-llama3-8b} reports the controlled-intervention results on LLaMA3-8B, matching the description in Section~\ref{sec:analysis}. Frequency-normalized loss (\emph{FreqNorm}) selectively suppresses the leading spike components, while intra-sentence token permutation (\emph{Shuffle}) selectively amplifies them; in both cases, changes rapidly vanish in the tail.

\begin{figure}[t]
    \captionsetup{skip=1pt}
    \centering
    \begin{minipage}{0.49\linewidth}
        \centering
        \includegraphics[width=\linewidth]{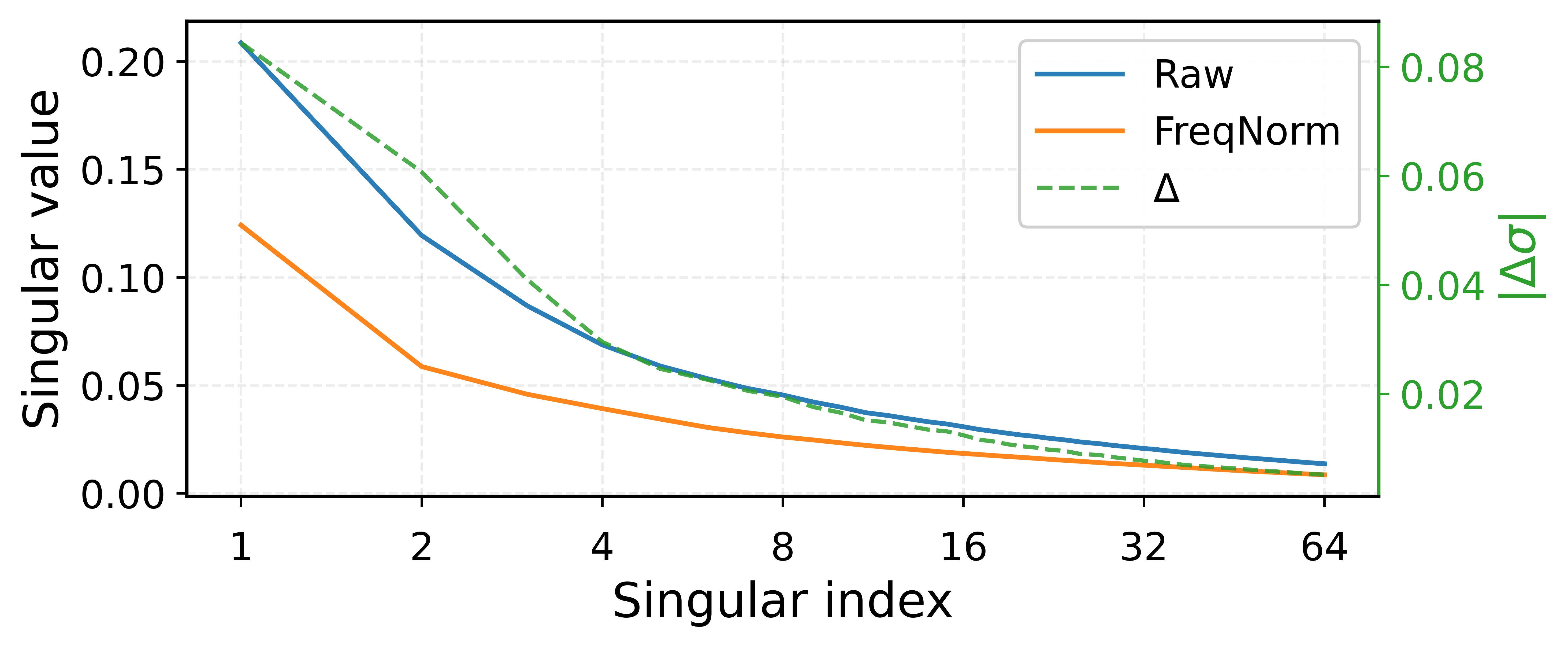}
    \end{minipage}
    \hfill
    \begin{minipage}{0.49\linewidth}
        \centering
        \includegraphics[width=\linewidth]{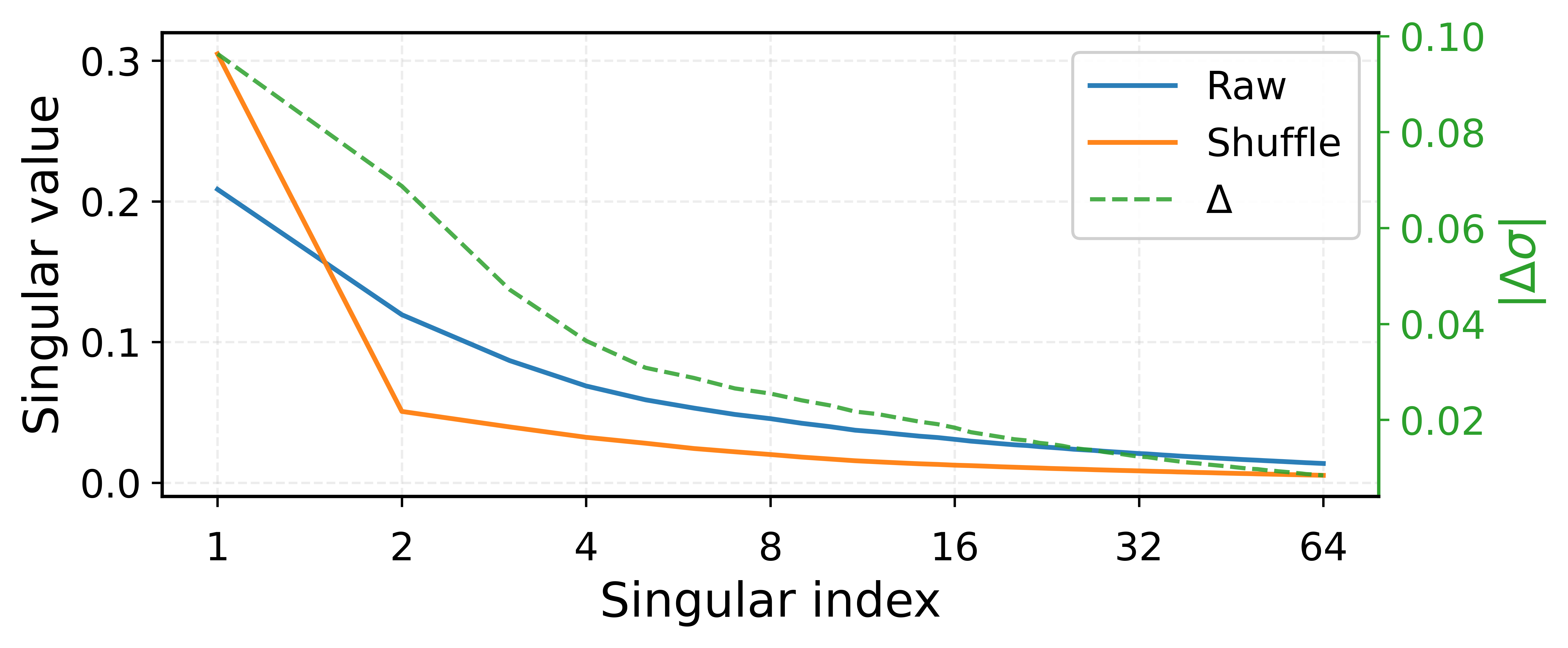}
    \end{minipage}
    \caption{Gradient spectrum under two controlled interventions on LLaMA3-8B. \textit{Left:} frequency-normalized loss (\emph{FreqNorm}) selectively suppresses the leading spike components. \textit{Right:} intra-sentence token permutation (\emph{Shuffle}) selectively amplifies them; in both cases, changes rapidly vanish in the tail.}
    \label{fig:spike-semantic-llama3-8b}
    \vspace{-1\baselineskip}
\end{figure}

\subsection{Spike Updating Suppresses Long-Tail Learning on LLaMA3-8B}
\label{appendix:spike-suppress-tail-8b}
Figure~\ref{fig:spike-suppress-tail-8b} reports the same analysis as Figure~\ref{fig:spike-suppress-tail}, but on LLaMA3-8B. We show (top) the cumulative spectral energy (CDF) of AdamW moments and (bottom) the distribution of tail updates under full normalization versus a tail-only baseline.

\begin{figure}[ht]
    \centering
    \begin{minipage}{0.49\linewidth}
        \centering
        \includegraphics[width=\linewidth]{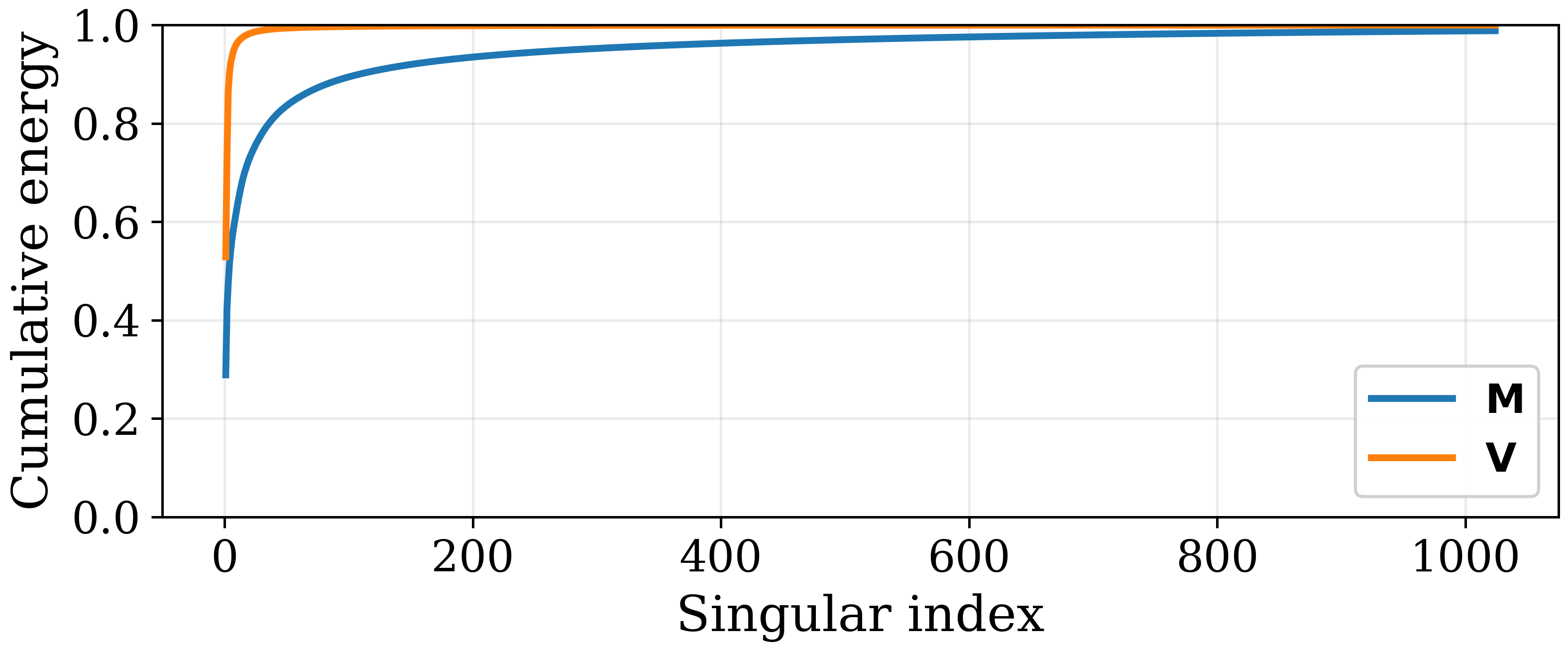}
    \end{minipage}
    \hfill
    \begin{minipage}{0.49\linewidth}
        \centering
        \includegraphics[width=\linewidth]{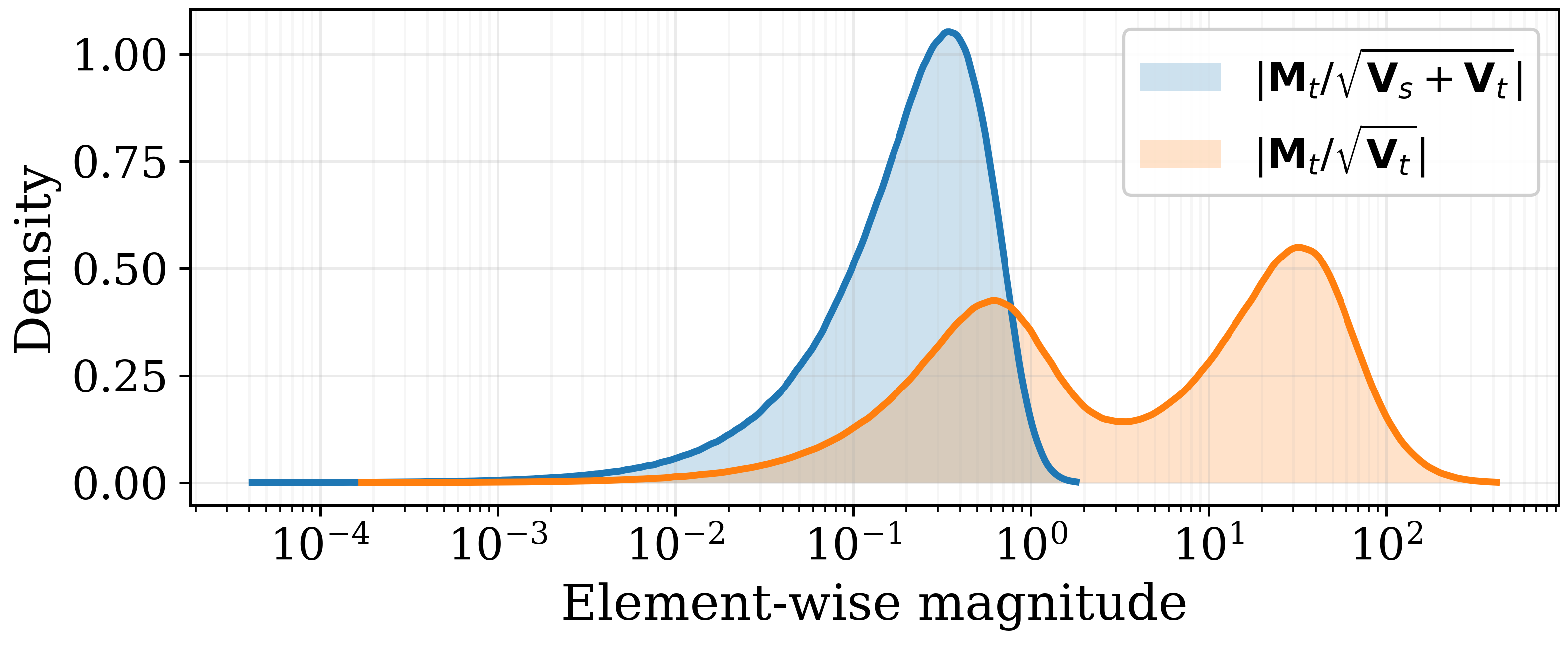}
    \end{minipage}
    \caption{{Spike-dominated second-moment accumulation suppresses tail updates (LLaMA3-8B).} \textit{Left:} cumulative spectral energy (CDF) of AdamW moments. \textit{Right:} element-wise magnitudes of tail updates under full normalization versus the tail-only baseline.}
    \label{fig:spike-suppress-tail-8b}
\end{figure}

\subsection{Proof of Theorem~\ref{thm:spike_lr_bound}}
\label{appendix:proof_spike_lr_bound}

\begin{proof}[Proof of Theorem~\ref{thm:spike_lr_bound}]
Let $\mathbf{w}\triangleq \mathrm{vec}(\mathbf{W})$ and let the mini-batch gradient be the random vector
$\mathbf{g}\triangleq \mathrm{vec}(\mathbf{G})$ with mean $\bar{\mathbf{g}}=\mathbb{E}[\mathbf{g}]$ and covariance
$\mathrm{Cov}(\mathbf{g})=\mathbf{\Sigma}/B$.
Consider the SGD update $\mathbf{w}^+=\mathbf{w}-\eta\,\mathbf{g}$ and denote the Hessian at the current iterate by
$\mathbf{H}\triangleq \nabla^2 L(\mathbf{w})$.

\paragraph{Step 1: second-order surrogate of $\mathbb{E}[L(\mathbf{w}^+)]$.}
Using the second-order Taylor expansion of $L$ around $\mathbf{w}$, we have
\begin{equation}
L(\mathbf{w}-\eta\mathbf{g})
~\approx~
L(\mathbf{w})
-\eta\,\nabla L(\mathbf{w})^\top \mathbf{g}
+\frac{1}{2}\eta^2\,\mathbf{g}^\top \mathbf{H}\mathbf{g}.
\label{eq:taylor_w_eta_g}
\end{equation}
Taking expectation over the mini-batch randomness and noting $\nabla L(\mathbf{w})=\bar{\mathbf{g}}$, we obtain
\begin{align}
\mathbb{E}\!\left[L(\mathbf{w}-\eta\mathbf{g})\right]
&\approx
L(\mathbf{w})
-\eta\,\bar{\mathbf{g}}^\top \mathbb{E}[\mathbf{g}]
+\frac{1}{2}\eta^2\,\mathbb{E}\!\left[\mathbf{g}^\top \mathbf{H}\mathbf{g}\right]
\nonumber\\
&=
L(\mathbf{w})
-\eta\,\|\bar{\mathbf{g}}\|_2^2
+\frac{1}{2}\eta^2\,\mathbb{E}\!\left[\mathbf{g}^\top \mathbf{H}\mathbf{g}\right].
\label{eq:expected_surrogate}
\end{align}
Moreover, for any random vector $\mathbf{g}$ with mean $\bar{\mathbf{g}}$ and covariance $\mathbf{\Sigma}/B$,
\begin{equation}
\mathbb{E}\!\left[\mathbf{g}^\top \mathbf{H}\mathbf{g}\right]
=
\mathrm{tr}\!\left(\mathbf{H}\,\mathbb{E}[\mathbf{g}\mathbf{g}^\top]\right)
=
\mathrm{tr}\!\left(\mathbf{H}\left(\bar{\mathbf{g}}\bar{\mathbf{g}}^\top + \frac{1}{B}\mathbf{\Sigma}\right)\right)
=
\bar{\mathbf{g}}^\top \mathbf{H}\bar{\mathbf{g}}
+\frac{1}{B}\mathrm{tr}(\mathbf{\Sigma}\mathbf{H}).
\label{eq:quad_term_eval}
\end{equation}
Substituting~\eqref{eq:quad_term_eval} into~\eqref{eq:expected_surrogate} yields a quadratic surrogate in $\eta$.

\paragraph{Step 2: minimizing the quadratic gives \eqref{eq:eta_star}.}
Assuming $\mathbf{H}\succeq \mathbf{0}$ and the quadratic coefficient is nonzero, the surrogate is convex in $\eta$.
Its minimizer is
\begin{equation}
\eta^\ast
=
\frac{\|\bar{\mathbf{g}}\|_2^2}{
\bar{\mathbf{g}}^\top \mathbf{H}\bar{\mathbf{g}}+\frac{1}{B}\mathrm{tr}(\mathbf{\Sigma}\mathbf{H})},
\end{equation}
which is exactly~\eqref{eq:eta_star}.

\paragraph{Step 3: spike-restricted variance yields the upper bounds in \eqref{eq:eta_star_spike_upper}.}
Let $\mathbf{\Pi}_k=\sum_{i=1}^k \mathbf{s}_i\mathbf{s}_i^\top$ be the projector onto the spike subspace and
$\mathbf{\Sigma}_s=\mathbf{\Pi}_k\mathbf{\Sigma}\mathbf{\Pi}_k$.
Since $\mathbf{H}\succeq \mathbf{0}$ and $\mathbf{\Pi}_k$ is an orthogonal projector, we have
$\mathbf{0}\preceq \mathbf{\Pi}_k\mathbf{H}\mathbf{\Pi}_k \preceq \mathbf{H}$.
Therefore, with $\mathbf{\Sigma}\succeq \mathbf{0}$,
\begin{equation}
\mathrm{tr}(\mathbf{\Sigma}_s\mathbf{H})
=\mathrm{tr}(\mathbf{\Pi}_k\mathbf{\Sigma}\mathbf{\Pi}_k\mathbf{H})
=\mathrm{tr}(\mathbf{\Sigma}\,\mathbf{\Pi}_k\mathbf{H}\mathbf{\Pi}_k)
\le \mathrm{tr}(\mathbf{\Sigma}\mathbf{H}).
\label{eq:trace_monotone}
\end{equation}
Plugging~\eqref{eq:trace_monotone} into the denominator of~\eqref{eq:eta_star} gives the first inequality in
\eqref{eq:eta_star_spike_upper}.
Finally, dropping the nonnegative term $\bar{\mathbf{g}}^\top \mathbf{H}\bar{\mathbf{g}}\ge 0$ yields the second inequality in
\eqref{eq:eta_star_spike_upper}.

\paragraph{Step 4: the $\mu$-strongly-convex bound \eqref{eq:eta_star_spike_upper_mu}.}
If $\mathbf{H}\succeq \mu \mathbf{I}$ for some $\mu>0$, then
\begin{equation}
\mathrm{tr}(\mathbf{\Sigma}_s\mathbf{H})
\ge
\mathrm{tr}(\mathbf{\Sigma}_s\,\mu\mathbf{I})
=
\mu\,\mathrm{tr}(\mathbf{\Sigma}_s).
\end{equation}
Hence
\begin{equation}
\eta^\ast
\le
\frac{B\,\|\bar{\mathbf{g}}\|_2^2}{\mu\,\mathrm{tr}(\mathbf{\Sigma}_s)}.
\end{equation}
Using $\mathbf{\Sigma}_s=\mathbf{\Pi}_k\mathbf{\Sigma}\mathbf{\Pi}_k$ and $\mathbf{\Pi}_k=\sum_{i=1}^k \mathbf{s}_i\mathbf{s}_i^\top$
with orthonormal $\{\mathbf{s}_i\}_{i=1}^k$, we further have
\begin{equation}
\mathrm{tr}(\mathbf{\Sigma}_s)
=
\mathrm{tr}(\mathbf{\Pi}_k\mathbf{\Sigma})
=
\sum_{i=1}^k \mathbf{s}_i^\top \mathbf{\Sigma}\mathbf{s}_i,
\end{equation}
which gives~\eqref{eq:eta_star_spike_upper_mu} and completes the proof.
\end{proof}

\subsection{Experiment details}
\label{appendix:experiment-details}

The detailed training configurations of Qwen3-0.6B and LLaMA3-8B are shown in Table~\ref{tab:qwen3_config} and Table~\ref{tab:llama3_config}.

\begin{table}[h]
\centering
\caption{Model configurations for Qwen3-0.6B.}
\label{tab:qwen3_config}
\begin{tabular}{lc}
\hline
\textbf{Configurations} & \textbf{Qwen3-0.6B} \\ \hline
Hidden activation & silu \\
Max position embeddings & 40960 \\
Vocabulary size & 151936 \\
Sequence length & 1024 \\
LayerNorm $\epsilon$ (rms\_norm\_eps) & $1 \times 10^{-6}$ \\
Dropout probability (attention\_dropout) & 0.0 \\
LR warm-up steps & 2,000 \\
LR scheduler & Cosine \\
Hidden size & 1024 \\
Intermediate size & 3072 \\
Hidden layers & 28 \\
Num. attention heads & 16 \\ 
Batch Size & 512 \\ \hline
\end{tabular}
\end{table}

\begin{table}[h]
\centering
\caption{Model configurations for LLaMA3-8B.}
\label{tab:llama3_config}
\begin{tabular}{lc}
\hline
\textbf{Configurations} & \textbf{LLaMA3-8B} \\ \hline
Hidden activation & silu \\
Max position embeddings & 8192 \\
Vocabulary size & 128256 \\
Sequence length & 4,096 \\
LayerNorm $\epsilon$ (rms\_norm\_eps) & $1 \times 10^{-5}$ \\
Dropout probability (attention\_dropout) & 0.0 \\
LR warm-up steps & 7629 \\
LR scheduler & Cosine \\
Hidden size & 4096 \\
Intermediate size & 14336 \\
Hidden layers & 32 \\
Num. attention heads & 32 \\ 
Batch Size & 512 \\ \hline
\end{tabular}
\end{table}


\end{document}